%% file: main.tex
\documentclass[a4paper]{article}
\usepackage[all]{xy}\usepackage[latin1]{inputenc}        
\usepackage[dvips]{graphics,graphicx}
\usepackage{amsfonts,amssymb,amsmath,color,mathrsfs, amstext}
\usepackage{amsbsy, amsopn, amscd, amsxtra, amsthm,authblk}
\usepackage{enumerate,algorithm,algpseudocode}
\usepackage{fancyhdr}
\usepackage{setspace}
\usepackage{ulem}
\usepackage{bbm}
\usepackage{subcaption}
\usepackage{graphicx}
\usepackage{upref,ulem}
\usepackage{geometry}
\usepackage{float}
\usepackage{yhmath}
\usepackage{booktabs}
\usepackage{graphicx}
\usepackage{multirow}
\usepackage{makecell}
\usepackage[colorlinks,
            linkcolor=red,
            anchorcolor=red,
            citecolor=red,
            urlcolor=black,
            ]{hyperref}

\numberwithin{equation}{section}

\newcommand \footnoteONLYtext[1]
{
	\let \mybackup \thefootnote
	\let \thefootnote \relax
	\footnotetext{#1}
	\let \thefootnote \mybackup
	\let \mybackup \imareallyundefinedcommand
}

\definecolor{DarkBlue}{rgb}{0,0,.8}
\newtheorem{theorem}{Theorem}[section]
\newtheorem{lemma}{Lemma}[section]
\newtheorem{definition}{Definition}[section]
\newtheorem{remark}{Remark}[section]
\newtheorem{proposition}{Proposition}[section]

\numberwithin{equation}{section}

\providecommand{\keywords}[1]
{
  \small	
  \textbf{\textit{Keywords:}} #1
}

\providecommand{\msc}[1]
{
  \small	
  \textbf{\textit{Mathematics Subject Classification:}} #1
}

\begin{document}

\title{A Generative Sampler for distributions with possible discrete parameter based on Reversibility}

\author[1,2]{Lei Li \thanks{leili2010@sjtu.edu.cn} }
\author[1]{Zhen Wang \thanks{wang-zhen@sjtu.edu.cn}}
\author[1]{Lishuo Zhang \thanks{ShawnLi9@sjtu.edu.cn; Corresponding author.}  }
\affil[1]{School of Mathematical Sciences, Shanghai Jiao Tong University, Shanghai, 200240, P.R.China.}
\affil[2]{Institute of Natural Sciences, MOE-LSC, Shanghai Jiao Tong University, Shanghai, 200240, P.R.China.}

\date{}
\maketitle

\input{section/abstract}

\input{section/introduction}
\input{section/preliminary}

\input{section/generative_framework_based_on_reversibility}
\input{section/theory}
\input{section/numerical_experiments}

\input{section/conclusion}


\appendix

\section{The surrogate gradient}\label{app:Nlimitofsurrogate}

In the generative framework, let $G_\theta:\mathbb{R}^d\to\mathbb{R}^d$ is a differentiable transport map from a base distribution $x\sim p_0$ to the target distribution $G_\theta(x)\sim p_\theta$. During training, we generate samples $\{X_i\}^N_{i=1}$ from $p_0$ via $X_i=G_\theta(Z_i)$ with $Z_i\sim p_0$. For each $X_i$, we then generate $Y_i$ by simulating one step of a Markov process with initial state $X_i$ and transition kernel $p(\cdot,\cdot)$. This yields paired samples $\{(X_i,Y_i)\}^N_{i=1}$. In numerical realization, the loss using the samples is given by
\begin{align*}
    \mathcal{L}^N(\theta)
    &=\frac{1}{N^2}\sum^N_{i,j=1}{\left[k((X_i,Y_i),(X_j,Y_j))-k((X_i,Y_i),(Y_j,X_j))\right]}\\
    &\quad-\frac{1}{N^2}\sum^N_{i,j=1}{\left[k((Y_i,X_i),(X_j,Y_j))-k((Y_i,X_i),(Y_j,X_j))\right]}.
\end{align*}
Only taking $\theta$ derivative for $X_i$, we obtain the following
\begin{align*}
    \hat{g}_N(\theta)
    &=\frac{1}{N^2}\sum^N_{i,j=1}{\left[\nabla_{x_1}k((X_i,Y_i),(X_j,Y_j))\cdot\frac{d}{d\theta}X_i+\nabla_{y_1}k((X_i,Y_i),(X_j,Y_j))\cdot\frac{d}{d\theta}X_j\right]}\\
    &\quad-\frac{1}{N^2}\sum^N_{i,j=1}{\left[\nabla_{x_1}k((X_i,Y_i),(Y_j,X_j))\cdot\frac{d}{d\theta}X_i+\nabla_{y_2}k((X_i,Y_i),(Y_j,X_j))\cdot\frac{d}{d\theta}X_j\right]}\\
    &\quad-\frac{1}{N^2}\sum^N_{i,j=1}{\left[\nabla_{x_2}k((Y_i,X_i),(X_j,Y_j))\cdot\frac{d}{d\theta}X_i+\nabla_{y_1}k((Y_i,X_i),(X_j,Y_j))\cdot\frac{d}{d\theta}X_j\right]}\\
    &\quad+\frac{1}{N^2}\sum^N_{i,j=1}{\left[\nabla_{x_2}k((Y_i,X_i),(Y_j,X_j))\cdot\frac{d}{d\theta}X_i+\nabla_{y_2}k((Y_i,X_i),(Y_j,X_j))\cdot\frac{d}{d\theta}X_j\right]},
\end{align*}
where $\nabla_{x_i}k((x_1,x_2),(y_1,y_2))$ denotes the gradient with respect to $x_i$, and $\nabla_{y_i}k((x_1,x_2),(y_1,y_2))$ the gradient with respect to $y_i$ for $i=1,2$. When the kernel $k$ is symmetric, the surrogate gradient we use is
\begin{align*}
    \hat{g}_N(\theta)
    &=\frac{2}{N^2}\sum^N_{i,j=1}{\left[\nabla_{x_1}k((X_i,Y_i),(X_j,Y_j))\cdot\frac{d}{d\theta}X_i-\nabla_{x_1}k((X_i,Y_i),(Y_j,X_j))\cdot\frac{d}{d\theta}X_i\right]}\\
    &\quad-\frac{2}{N^2}\sum^N_{i,j=1}{\left[\nabla_{x_2}k((Y_i,X_i),(X_j,Y_j))\cdot\frac{d}{d\theta}X_i-\nabla_{x_2}k((Y_i,X_i),(Y_j,X_j))\cdot\frac{d}{d\theta}X_i\right]},
\end{align*}
which is \eqref{eq:surrogate_grad}. Hence, for the surrogate gradient, whether applying the Monte Carlo approximation first or not gives the same result. We then define the empirical distribution
\begin{equation*}
    \mu^N_1(\vec{x}):=\frac{1}{N}\sum^N_{i=1}{\delta_{(X_i,Y_i)}(\vec{x})},\qquad\mu^N_2(\vec{x}):=\frac{1}{N}\sum^N_{i=1}{\delta_{(Y_i,X_i)}(\vec{x})},
\end{equation*}
and we further assume
\begin{equation*}
    \frac{d}{d\theta}X_i=\frac{d}{d\theta}G_\theta(Z_i)=v(\theta,G_\theta(Z_i)).
\end{equation*}
It follows that
\begin{align*}
    \hat{g}_N(\theta)
    &=2\int_{\mathbb{R}^{2d}}{\int_{\mathbb{R}^{2d}}{v(\theta,x_1)\cdot\nabla_{x_1}[k(\vec{x},\vec{y})]\mu^N_1(\vec{x})(\mu^N_1(\vec{y})-\mu^N_2(\vec{y}))}d\vec{x}}d\vec{y}\\
    &\quad-2\int_{\mathbb{R}^{2d}}{\int_{\mathbb{R}^{2d}}{v(\theta,x_2)\cdot\nabla_{x_2}[k(\vec{x},\vec{y})]\mu^N_2(\vec{x})(\mu^N_1(\vec{y})-\mu^N_2(\vec{y}))}d\vec{x}}d\vec{y}.
\end{align*}
Let $N\rightarrow +\infty$ and we obtain \eqref{eq:inf_surrogate_grad},
\begin{align*}
    \hat{g}_\infty(\theta)
    &=2\int_{\mathbb{R}^{2d}}{\int_{\mathbb{R}^{2d}}{v(\theta,x_1)\cdot\nabla_{x_1}[k(\vec{x},\vec{y})]\mu_1(\vec{x})(\mu_1(\vec{y})-\mu_2(\vec{y}))}d\vec{x}}d\vec{y}\\
    &\quad-2\int_{\mathbb{R}^{2d}}{\int_{\mathbb{R}^{2d}}{v(\theta,x_2)\cdot\nabla_{x_2}[k(\vec{x},\vec{y})]\mu_2(\vec{x})(\mu_1(\vec{y})-\mu_2(\vec{y}))}d\vec{x}}d\vec{y},
\end{align*}
where $\mu_1(\vec{x})=p_\theta(x_1)p(x_1,x_2)$ and $\mu_2(\vec{x})=p_\theta(x_2)p(x_2,x_1)$.



\bibliographystyle{plain}
\bibliography{Revers}

\end{document}

%% file: section/abstract.tex
\begin{abstract}
Learning to sample from complex unnormalized distributions is a fundamental challenge in computational physics and machine learning. While score-based and variational methods have achieved success in continuous domains, extending them to discrete or mixed-variable systems remains difficult due to ill-defined gradients or high variance in estimators. We propose a unified, target-gradient-free generative sampling framework applicable across diverse state spaces. Building on the fact that detailed balance implies the time-reversibility of the equilibrium stochastic process, we enforce this symmetry as a statistical constraint. Specifically, using a prescribed physical transition kernel (such as Metropolis-Hastings), we minimize the Maximum Mean Discrepancy (MMD) between the joint distributions of forward and backward Markov trajectories. Crucially, this training procedure relies solely on energy evaluations via acceptance ratios, circumventing the need for target score functions or continuous relaxations. We demonstrate the versatility of our method on three distinct benchmarks: (1) a continuous multi-modal Gaussian mixture, (2) the discrete high-dimensional Ising model, and (3) a challenging hybrid system coupling discrete indices with continuous dynamics. Experiments show that our framework accurately reproduces thermodynamic observables and captures mode-switching behavior across all regimes, offering a physically grounded and universally applicable alternative for equilibrium sampling.
\end{abstract}

\keywords{Ising model; discrete generative model; detailed balance; reversibility}\\

\msc{49Q22; 68T07}

%% file: section/introduction.tex
\section{Introduction}
\label{sec:introduction}
Efficient sampling from high-dimensional distributions is a fundamental problem in computational physics and machine learning. While standard methods are well-established for continuous systems, sampling in discrete or mixed-variable domains remains challenging, particularly in equilibrium statistical mechanics of spin systems \cite{landau2021guide, newman1999monte}. The primary goal is to generate $\mathbf{s}$ from the Boltzmann distribution $p(\mathbf{s}) \propto \exp\!\big(-\beta H(\mathbf{s})\big)$, where $H(\mathbf{s})$ is the Hamiltonian and $\beta$ is the inverse temperature. While Markov Chain Monte Carlo (MCMC) algorithms, such as Metropolis-Hastings updates \cite{metropolis1953equation, hastings1970monte} or Glauber dynamics \cite{glauber1963time} are asymptotically exact, their local nature imposes a severe bottleneck near phase transitions: integrated autocorrelation times increase sharply with system size, a phenomenon commonly referred to as critical slowing down \cite{sokal1997monte, wolff1989collective}. This limitation has long motivated the search for efficient global sampling strategies.

Recent progress in machine learning has made deep generative models a practical tool for constructing fast surrogate samplers in computational physics. 
Rather than updating configurations through local Markov moves alone, these approaches introduce a parametric proposal distribution $q_\theta(\mathbf{s})$ that aims to approximate the target Boltzmann distribution $p(\mathbf{s})$. 
Representative examples include variational autoregressive schemes (VANs) \cite{wu2019solving} and normalizing-flow-based Boltzmann generators \cite{noe2019boltzmann}. 
Once trained, $q_\theta$ can generate approximately independent proposals at low marginal cost, which can substantially reduce autocorrelation in downstream estimation. 
A common training objective is the reverse Kullback--Leibler divergence,
\begin{equation}
\mathcal{L}(\theta)=\mathrm{KL}\!\left(q_\theta(\mathbf{s})\,\|\,p(\mathbf{s})\right),
\end{equation}
which is equivalent to minimizing the variational free energy\cite{wu2019solving}. 
This variational route has been particularly effective in continuous-state settings, including molecular simulations and related applications in particle physics \cite{noe2019boltzmann,kohler2021smooth,cranmer2020frontier}.

Extending flow-based generative models to discrete or hybrid systems presents inherent structural challenges. 
Normalizing flows\cite{rezende2015variational, jing2024machine} are built as differentiable bijections over continuous variables \cite{papamakarios2021normalizing}; applying them to discrete states typically requires dequantization or continuous relaxations to obtain a tractable likelihood and enable optimization \cite{ho2019flow++, hoogeboom2020learning}. 
These modifications can introduce modeling bias and make performance sensitive to the chosen relaxation scheme and hyperparameters. 
Furthermore, in hybrid systems where discrete and continuous variables are coupled, gradient-based methods face a dilemma: score functions are ill-defined for the discrete components, while relaxation techniques may fail to capture the sharp dependencies between modes.
More generally, training discrete generative models often relies on stochastic gradient estimators such as REINFORCE \cite{williams1992simple} or continuous reparameterizations such as the Gumbel--Softmax/Concrete relaxations \cite{jang2016categorical, maddison2016concrete}, which may suffer from high variance or non-negligible bias. 
This motivates the development of training objectives that operate directly on the discrete or mixed state space while enforcing the physical constraints of the target equilibrium distribution, without requiring gradients of the energy.

In this paper, we develop a target-gradient-free generative sampling framework for equilibrium distributions based on reversibility. 
Our method bypasses variational free-energy minimization and score-based training, and instead builds on a fundamental physical principle: the time-reversibility of the equilibrium stochastic process.

Concretely, we fix a physical transition kernel $p(\mathbf{s}, \mathbf{s'})$ (e.g., a Metropolis--Hastings step) whose stationary distribution is the target Boltzmann law,and couple it with a neural generator $G_\theta: \mathcal{Z} \to \mathcal{S}$. We denote the induced marginal distribution of the generated samples as $p_\theta$, where $\mathbf{s} = G_\theta(\mathbf{z})$ with $\mathbf{z} \sim \rho_0$. 
Starting from $\mathbf{s} \sim p_\theta$, we generate one-step transitions $\mathbf{s'} \sim p(\mathbf{s},\cdot)$, which induces a joint distribution over pairs
\[
\mu_\theta(\mathbf{s},\mathbf{s'}) := G_\theta(\mathbf{s})\,p(\mathbf{s},\mathbf{s'}).
\]
If $p_\theta$ matches the target equilibrium distribution, then the resulting joint distribution is symmetric under time reversal, i.e., $\mu_\theta(\mathbf{s},\mathbf{s'}) = \mu_\theta(\mathbf{s'},\mathbf{s})$. 
We therefore train $q_\theta$ by penalizing deviations from this symmetry using a Maximum Mean Discrepancy (MMD) loss\cite{gretton2012kernel}:
\begin{equation}
\mathcal{L}(\theta)
=\mathrm{MMD}^2\!\Big(\mu_\theta,\; \mu_\theta \circ \tau^{-1}\Big),
\qquad \tau(\mathbf{s},\mathbf{s'})=(\mathbf{s'},\mathbf{s}).
\label{eq:mmd_reversibility}
\end{equation}
In practice, this compares samples from the forward joint distribution $(\mathbf{s}, \mathbf{s}')$ with their time-reversed counterparts $(\mathbf{s}', \mathbf{s})$ and drives the generator toward a reversible equilibrium sampler. Because this alignment is enforced strictly through the local transition dynamics, training requires only energy evaluations through acceptance ratios (i.e., $\Delta H$), while gradients are taken solely with respect to the network parameters.

Overall, the proposed framework offers the following advantages.
\begin{itemize}
\item \textbf{Data-free Training.} Training does not require samples from the target Boltzmann distribution. It only needs access to the unnormalized density (energy) through density ratios / energy differences used in the acceptance rule (e.g., $\Delta H$), hence bypassing the need for pre-generated equilibrium datasets.

\item \textbf{Target-gradient-free.} Unlike score-based models\cite{song2020score} or Langevin-type dynamics\cite{roberts1996exponential} that require $\nabla_s H(s)$ (ill-defined for discrete states), our method relies only on energy differences used in the MCMC acceptance rule. Optimizing over $\theta$ via standard backpropagation is only performed on the gradient of generator and does not need the gradient of the target distribution or the energy function $H(s)$.

\item \textbf{Direct sampling after training.} Once trained, the generator produces independent draws from $q_\theta$ without running long Markov chains. This completely avoids the serial autocorrelation that plagues the sampling stage and enables massive parallelization, which is particularly advantageous near phase transitions where traditional MCMC methods are severely bottlenecked by critical slowing down.

\item \textbf{Jacobian-free generative modeling.} 
Standard normalizing flows rely on the change-of-variables formula\cite{rezende2015variational}, which necessitates the computation of Jacobian determinants\cite{dinh2016density} to track volume changes. 
This operation is not applicable for discrete variables\cite{hoogeboom2019integer}, creating a fundamental barrier for distributions with discrete parameters (like the spin systems). Our framework eliminates this requirement entirely by enforcing thermodynamic reversibility, allowing for direct training on discrete or mixed domains without the need for differentiable bijections.
\end{itemize}

The rest of the paper is organized as follows. 
Section~\ref{sec:preliminaries} establishes the mathematical foundations of equilibrium sampling, reviewing MCMC dynamics and the detailed balance condition. 
Section~\ref{sec:method} details our target-gradient-free generative framework, introducing the reversibility-based objective and its MMD implementation. 
Section~\ref{sec:theory} provides a theoretical analysis of the proposed method, establishing the weak convergence of the generated distribution to the target Boltzmann equilibrium under the minimization of the reversibility violation. 
Section~\ref{sec:experiments} demonstrates the versatility of our framework through numerical experiments on three distinct benchmarks: a multi-modal Gaussian mixture, the two-dimensional Ising model, and a coupled hybrid system. 
Finally, Section~\ref{sec:conclusion} concludes with a summary and potential extensions.

%% file: section/preliminary.tex
\section{Preliminaries}\label{sec:preliminaries}

In this section, we review the essential concepts and mathematical tools that underpin our framework. Specifically, we discuss the detailed balance condition for Markov chains, the Metropolis-Hastings algorithm, and the Maximum Mean Discrepancy (MMD).

\subsection{Time Reversibility and Detailed Balance}

We begin by recalling the general notion of time reversibility. Let $\{S_t\}_{t\in\mathcal{T}}$ be a stochastic process taking values in a state space $\mathcal{S}$, indexed by a discrete time set $\mathcal{T} \subseteq \mathbb{Z}$.

\begin{definition}[Reversible Stochastic Process]\cite{kelly1979reversibility}
\label{def:reversibility}
A stochastic process $\{S_t\}_{t\in\mathcal{T}}$ is said to be \textit{time-reversible} if, for any finite sequence of time points $t_1 < t_2 < \dots < t_n$ and any time shift $\tau$, the joint probability distributions satisfy:
\begin{equation}
\label{eq:proc-reversible}
  (S_{t_1}, S_{t_2}, \dots, S_{t_n})
  \;\stackrel{d}{=}\;
  (S_{\tau - t_1}, S_{\tau - t_2}, \dots, S_{\tau - t_n}),
\end{equation}
where $\stackrel{d}{=}$ denotes equality in distribution. Intuitively, the statistical properties of the process are invariant under the reversal of the time axis.
\end{definition}

In this work, we focus on time-homogeneous Markov chains on a state space $\mathcal{S}$ for which the state could have discrete components, characterized by a transition kernel $p(s, s') = \mathbb{P}(S_{t+1}=s' \mid S_t=s)$.
A probability measure $\pi$ on $\mathcal{S}$ is called a \textit{stationary distribution} if it satisfies the global balance equation:
\begin{equation}
\label{eq:stationary}
  \pi(s') \;=\; \sum_{s\in\mathcal{S}} \pi(s)\,p(s,s'),
  \qquad \forall\,s'\in\mathcal{S}.
\end{equation}

While stationarity ensures time-invariance, it does not strictly imply reversibility. A stronger condition, known as detailed balance, is sufficient to ensure thermodynamic equilibrium.

\begin{definition}[Detailed Balance\cite{binder1992monte} ]
\label{def:detailed_balance}
A probability distribution $\pi(s)$ and a transition kernel $p(s,s')$ satisfy the \textit{detailed balance} condition if the probability flux between any pair of states is balanced:
\begin{equation}
    \pi(s) p(s,s') = \pi(s') p(s',s), \quad \forall s, s' \in \mathcal{S}.
    \label{eq:detailed_balance}
\end{equation}
\end{definition}

Detailed balance is a sufficient condition for stationarity. The following proposition establishes the crucial link to time-reversibility, serving as the theoretical foundation for our objective.

\begin{proposition}[Stationarity and Symmetry\cite{swain1984handbook}]
\label{prop:reversibility}
If a distribution $\pi$ and a kernel $p$ satisfy the detailed balance condition (Eq.~\ref{eq:detailed_balance}), then:
\begin{enumerate}
    \item $\pi$ is a stationary distribution of the Markov chain defined by $p$.
    \item The chain initialized with $S_t \sim \pi$ is time-reversible. Specifically, the joint distribution of consecutive states is symmetric:
    \begin{equation}
        p(S_t = s, S_{t+1} = s') = p(S_t = s', S_{t+1} = s).
        \label{eq:pairwise_symmetry}
    \end{equation}
\end{enumerate}
\end{proposition}

\subsection{Markov Chain Monte Carlo}
\label{subsec:mcmc}

In statistical physics, direct sampling from the Boltzmann distribution $\pi(s) \propto e^{-\beta H(s)}$ is often infeasible due to the intractability of the partition function. Markov Chain Monte Carlo (MCMC) methods address this by constructing an ergodic Markov chain $\{S_t\}_{t \ge 0}$ that admits $\pi$ as its unique stationary distribution\cite{robert2004monte}.

Formally, for a Markov chain on the state space $\mathcal{S}$, if the chain is ergodic and the kernel $p$ satisfies the stationarity condition (Eq.~\ref{eq:stationary}), the law of the states then converges to the unique invariant distribution $\pi$ independently of the initial condition. Consequently, by Birkhoff's ergodic theorem\cite{walters2000introduction}, for any observable $f : \mathcal{S} \to \mathbb{R}$, the time average provides a strongly consistent estimator of the equilibrium expectation, converging almost surely:
\begin{equation}
\frac{1}{T} \sum_{t=1}^{T} f(S_t)
\;\xrightarrow[T \to \infty]{\text{a.s.}}\;
\mathbb{E}_{\pi}[f]
\;=\;
\sum_{s \in \mathcal{S}} f(s)\,\pi(s).
\label{eq:mcmc-ergodic}
\end{equation}
With this fact, one can then use the states of the chain at different time steps as the samples from the invaraint distribution $\pi$.

\paragraph{Metropolis--Hastings Algorithm.}\cite{metropolis1953equation, hastings1970monte}
A standard strategy to construct such kernels is the Metropolis--Hastings (MH) algorithm. 
Given a proposal distribution $q(s,s')$, a proposed move is accepted with probability:
\begin{equation}
\alpha(s \to s')
=
\min\!\left\{
1,\;
\frac{\pi(s')\,q(s,s')}{\pi(s)\,q(s',s)}
\right\}.
\label{eq:mh-accept}
\end{equation}
The resulting transition kernel $p$ satisfies detailed balance and hence preserves $\pi$.
In the common case of symmetric proposals $q(s',s) = q(s,s')$ (e.g., single-spin flips), the acceptance rule simplifies to the classical Metropolis criterion $\alpha(s \to s') = \min\{1, \pi(s')/\pi(s)\}$, which depends only on the energy difference\cite{landau2021guide}.

%% file: section/generative_framework_based_on_reversibility.tex
\section{Generative Framework Based on Reversibility}
\label{sec:method}

In this section, we propose RevGen, a general training framework for generative sampling based on time-reversibility. Our primary objective is to learn a parametric generator $G_\theta(\mathbf{s})$ that approximates a target distribution $\pi(\mathbf{s})$ defined on a state space $\mathcal{X}$.

Crucially, our framework is agnostic to the nature of the state space, applicable whether $\mathcal{X}$ is discrete (e.g., lattice spin configurations), continuous (e.g., vectors in $\mathbb{R}^d$), or a hybrid of both. Furthermore, the method assumes only that $\pi(\mathbf{s})$ is known up to a multiplicative constant, requiring access only to the energy function (or probability ratios) rather than gradients of the target density. Instead of minimizing variational free energy, we construct a training objective based on the universal physical principle of detailed balance.

\subsection{Markovian Coupling and Time-Reversibility}
\label{subsec:coupling}

Let $p(\mathbf{s},\mathbf{s}')$ be a prescribed transition kernel that satisfies the detailed balance condition \eqref{eq:detailed_balance} with respect to the target distribution $\pi(\mathbf{s})$:
\begin{equation}
    \pi(\mathbf{s}) p(\mathbf{s},\mathbf{s}') = \pi(\mathbf{s}') p(\mathbf{s}',\mathbf{s}).
\end{equation}
The kernel $p$ acts as a "physical oracle" and can be chosen flexibly. For continuous systems, it typically takes the form of a Gaussian random walk Metropolis step; for discrete systems, it may represent a Metropolis update with local spin flips.

We define a forward sampling process by first drawing a configuration from the neural generator, $\mathbf{s} \sim p_\theta(\mathbf{s})$, and then evolving it under the prescribed kernel for $m \ge 1$ steps, yielding $\mathbf{s}' \sim p^{(m)}(\mathbf{s},\cdot)$. This procedure induces a joint distribution over state pairs, denoted as $\mu_\theta(\mathbf{s}, \mathbf{s}')$:
\begin{equation}
    \mu_\theta(\mathbf{s}, \mathbf{s}') = p_\theta(\mathbf{s}) \, p^{(m)}(\mathbf{s}, \mathbf{s}').
    \label{eq:joint_dist}
\end{equation}
A fundamental property of equilibrium dynamics is that detailed balance is equivalent to the time-reversal symmetry of this joint distribution. If the generator perfectly recovers the target equilibrium (i.e., $p_\theta \equiv p$), substituting Eq.~\eqref{eq:detailed_balance} into Eq.~\eqref{eq:joint_dist} implies:
\begin{equation}
    \mu_\theta(\mathbf{s}, \mathbf{s}') = \mu_\theta(\mathbf{s}', \mathbf{s}).
    \label{eq:symmetry_condition}
\end{equation}
Accordingly, the discrepancy between the forward pair distribution $\mu_\theta(\mathbf{s},\mathbf{s}')$ and its time-reversed counterpart $\mu_\theta(\mathbf{s}',\mathbf{s})$ serves as a statistical indicator of deviation from equilibrium.

\subsection{The Reversibility-Based Objective via MMD}
\label{subsec:mmd_objective}

Since the analytical form of the joint density $\mu_\theta$ is intractable, we employ the Maximum Mean Discrepancy (MMD)~\cite{gretton2012kernel} to quantify the violation of the symmetry in Eq.~\eqref{eq:symmetry_condition}. MMD is a kernel-based metric that measures the distance between distributions based on samples, without requiring density estimation.

We define the loss function as the squared MMD distance between the forward pair distribution $\mu_\theta$ and the backward (swapped) distribution $\mu^\dagger_\theta$, where pairs are flipped $(\mathbf{s}, \mathbf{s}') \to (\mathbf{s}', \mathbf{s})$. 
Given a batch of $N$ samples $\{(\mathbf{s}_i, \mathbf{s}'_i)\}_{i=1}^N$ generated by the forward process, we construct the forward set $\mathcal{D}_{\text{fwd}} = \{\mathbf{x}_i\}_{i=1}^N$ with $\mathbf{x}_i = (\mathbf{s}_i, \mathbf{s}'_i)$, and the backward set $\mathcal{D}_{\text{bwd}} = \{\mathbf{y}_i\}_{i=1}^N$ with $\mathbf{y}_i = (\mathbf{s}'_i, \mathbf{s}_i)$. The empirical training objective is formulated using the biased $V$-statistic of the Maximum Mean Discrepancy, which is computationally efficient for large-scale mini-batch optimization. Specifically, given $N$ pairs of forward and backward samples $\{(\mathbf{x}_i, \mathbf{y}_i)\}_{i=1}^N$, the loss is computed as:
\begin{equation}
    \mathcal{L}(\theta) = \frac{1}{N^2} \sum_{i, j=1}^N k(\mathbf{x}_i, \mathbf{x}_j) + \frac{1}{N^2} \sum_{i, j=1}^N k(\mathbf{y}_i, \mathbf{y}_j) - \frac{2}{N^2} \sum_{i, j=1}^N k(\mathbf{x}_i, \mathbf{y}_j),
    \label{eq:empirical_loss}
\end{equation}
where $k(\cdot, \cdot)$ denotes the chosen kernel function. 

By utilizing the $V$-statistic, we can leverage highly optimized matrix operations in PyTorch, where the loss is equivalent to the mean of the kernel Gram matrices. This estimator remains asymptotically consistent as the batch size $N$ increases, providing stable gradients for the generator's convergence.
Here, the kernel function $k(\cdot, \cdot)$ is defined on the product space $\mathcal{X} \times \mathcal{X}$. We choose $k$ to be characteristic (e.g., a Gaussian RBF kernel for continuous spaces or a Hamming-distance kernel for discrete spaces\cite{kondor2002diffusion, sejdinovic2013equivalence}) to ensure that $\mathcal{L}(\theta)=0$ if and only if detailed balance is strictly satisfied.

\subsection{Optimization Dynamics and Gradient Analysis}
\label{subsec:optimization_dynamics}

The generator parameters $\theta$ are optimized via Adam\cite{kingma2014adam} by minimizing $\mathcal{L}(\theta)$. A critical structural feature of our framework is the gradient flow during backpropagation. The generation of a pair $(\mathbf{s}, \mathbf{s}')$ involves two stages:
\begin{enumerate}
    \item \textbf{Generation ($\mathbf{s} \sim p_\theta$):} This step is fully differentiable (e.g., via the reparameterization trick\cite{kingma2013auto,rezende2014stochastic}), establishing a direct path for gradients from the sample $\mathbf{s}$ to $\theta$.
    \item \textbf{Transition ($\mathbf{s}' \sim p(\mathbf{s},\cdot)$):} This step involves stochastic acceptance/rejection dynamics, which are non-differentiable.
\end{enumerate}

To stabilize training, we enforce a \textbf{stop-gradient} condition on the transition step. The updated configuration $\mathbf{s}'$ is treated as a fixed target detached from the computation graph. Consequently, rather than computing the exact intractable gradient, we define a surrogate gradient estimator $\hat{\mathbf{g}}_\theta$ derived solely by differentiating the loss with respect to the generator's direct output $\mathbf{s}$:
\begin{equation}
    \hat{\mathbf{g}}_\theta
    := \mathbb{E}_{\mathbf{s}, \mathbf{s}'}
    \left[
    \nabla_{\mathbf{s}} \mathrm{MMD}^2(\dots) \bigg|_{\mathbf{s}' \text{ fixed}} \cdot \nabla_\theta \mathbf{s}
    \right] \approx \nabla_\theta \mathcal{L}(\theta).
    \label{eq:grad_stop_kernel}
\end{equation}

By detaching $\mathbf{s}'$, we treat the MCMC transition not as a differentiable layer, but as a fixed operator that reveals the local direction toward the equilibrium manifold. This effectively implements a semi-gradient update via $\hat{\mathbf{g}}_\theta$: the generator attempts to move $\mathbf{s}$ such that the forward pair $(\mathbf{s}, \mathbf{s}')$ becomes statistically indistinguishable from its time-reversed counterpart $(\mathbf{s}', \mathbf{s})$, using $\mathbf{s}'$ as a physically grounded anchor.

\begin{remark}[Framework Extensibility]
    While we demonstrate the framework using MMD and gradient-based learning, the underlying principle of time-reversibility is inherently implementation-agnostic. For instance, one may adopt a weak formulation \cite{cai2024weak,ZANG2020109409}, in which the fixed MMD metric is replaced by a learnable discriminator $f_\phi(\mathbf{s}, \mathbf{s}')$ parameterized by a neural network. Instead of relying on a static RKHS kernel, this discriminator acts as an adaptive metric trained to maximize the discrepancy between the forward trajectory distribution $\mu_F(\mathbf{s}, \mathbf{s}') = p_\theta(\mathbf{s}) p(\mathbf{s},\mathbf{s}')$ and the time-reversed backward flow $\mu_B(\mathbf{s}, \mathbf{s}') = p_\theta(\mathbf{s}') p(\mathbf{s}', \mathbf{s})$. The framework then transforms into a weak adversarial minimax game:
    \begin{equation}
        \min_{\theta} \max_{\phi} \; \mathbb{E}_{(\mathbf{s}, \mathbf{s}') \sim P_F}[f_\phi(\mathbf{s}, \mathbf{s}')] - \mathbb{E}_{(\mathbf{s}', \mathbf{s}) \sim P_B}[f_\phi(\mathbf{s}', \mathbf{s})].
    \end{equation}
    In this adversarial regime, the discriminator dynamically learns optimal feature representations to penalize detailed balance violations. Furthermore, to optimize the generator without requiring differentiable relaxations (e.g., Gumbel-Softmax) for the discrete sampling steps, the optimization could alternatively employ derivative-free methods \cite{larson2019derivative}, such as Evolutionary Strategies (ES) \cite{salimans2017evolution}, score-function estimators (REINFORCE) \cite{lievin2020optimal} or consensus-based optimization method (CBO) \cite{pinnau2017consensus,carrillo2021consensus}, to directly traverse the non-differentiable discrete landscape.
\end{remark}

\subsection{Training Algorithm}
\label{sec:algorithm}
The training procedure, summarized in Algorithm~\ref{alg:training}, iterates between neural generation, physical simulation, and parameter updates.

\begin{algorithm}[H]
    \caption{Generative Sampling via Reversibility (RevGen)}
    \label{alg:training}
    \begin{algorithmic}[1]
        \Require Target energy $E(\mathbf{s})$, Base noise distribution $\rho_0$
        \Require Transition Kernel $p$ (e.g., Metropolis), Steps $m$
        \Require Batch size $N$, Max iterations $N_{I}$
        
        \State Initialize generator parameters $\theta$.
        
        \For{$n = 1$ to $N_{I}$}
            \State \textbf{1. Generation:}
            \State Sample noise $\{\mathbf{z}_i\}_{i=1}^N \sim \rho_0$.
            \State Generate states $\mathbf{s}_i = G_\theta(\mathbf{z}_i)$.
            
            \State \textbf{2. Transition (Simulation):}
            \For{$i = 1$ to $N$}
                \State Sample $\mathbf{s}'_i \sim p^m(\mathbf{s}_i,\cdot)$ 
                \State $\mathbf{s}'_i \leftarrow \text{stop\_gradient}(\mathbf{s}'_i)$ \Comment{Detach from graph}
            \EndFor
            
      \State \textbf{3. Optimization:}
            \State Construct pairs: $\mathcal{D}_{\text{fwd}} = \{(\mathbf{s}_i, \mathbf{s}'_i)\}$, $\mathcal{D}_{\text{bwd}} = \{(\mathbf{s}'_i, \mathbf{s}_i)\}$
            \State Compute Loss $\mathcal{L} = \widehat{\text{MMD}}^2(\mathcal{D}_{\text{fwd}}, \mathcal{D}_{\text{bwd}})$ and derive surrogate gradient $\hat{\mathbf{g}}_\theta$ via Eq.~\eqref{eq:grad_stop_kernel}.
            \State Update parameters: $\theta \leftarrow \text{Adam}(\theta, \hat{\mathbf{g}}_\theta)$.
        \EndFor
        \Ensure Trained generator $G_\theta$
    \end{algorithmic}
\end{algorithm}

\subsection{Generator Architectures}
\label{subsec:architectures}

The proposed framework is compatible with any differentiable generator architecture $G_\theta: \mathcal{Z} \to \mathcal{X}$. We instantiate specific architectures based on the data topology.

\paragraph{Continuous Systems.} 
For continuous state spaces, we employ deep generative networks (e.g., Normalizing Flows \cite{rezende2015variational} or ResNets \cite{he2016deep}) that map a base noise distribution $\rho_0$ to the target space $\mathbb{R}^d$. Unlike likelihood-based methods, our MMD objective does not require invertible mappings or tractable Jacobian determinants, allowing for a flexible choice of architectures provided they are differentiable.

\paragraph{Discrete Systems.} 
The generator is parameterized by an MLP that maps continuous latent noise $\mathbf{z}$ to real-valued logits $\mathbf{h} = \text{MLP}(\mathbf{z})$, which are subsequently mapped to physical states $\mathbf{s} \in \{-1, +1\}^d$ via a discrete sampling layer (e.g., the $\text{sign}(\cdot)$ function or categorical sampling). 
As emphasized in Section~\ref{sec:introduction}, our fundamental aim is to operate the training objective directly on the strictly discrete state space without requiring continuous reparameterization of the target density. 
Therefore, we clarify that the continuous relaxation employed here---specifically, the Straight-Through Estimator (STE) \cite{bengio2013estimating} used to approximate gradients during the backward pass---is exclusively a computational tool for the network's optimization procedure, rather than a modification to the training loss itself. 
The reversibility-based MMD objective remains rigorously discrete. In fact, if one wishes to avoid continuous relaxations entirely, this exact discrete objective can be natively optimized using derivative-free optimization methods.

\paragraph{Hybrid Systems.} 
For systems comprising coupled continuous variables $\mathbf{x}$ and discrete indices $k$, we employ a multi-head architecture. A shared MLP backbone first maps the latent noise $\mathbf{z}$ to a common feature representation. This representation is then bifurcated into two specific heads:
(1) A continuous head that directly outputs the continuous variables $\mathbf{x} \in \mathbb{R}^d$;
(2) A discrete head that outputs logits $\mathbf{h}$ for the discrete indices. The final discrete state $k$ is obtained by sampling from the categorical distribution parameterized by $\mathbf{h}$. 
This shared design allows the generator to capture the joint dependencies between the continuous dynamics and the discrete modes within a unified latent representation.

%% file: section/theory.tex
\section{Theoretical Analysis}
\label{sec:theory}

In this section, we provide a theoretical analysis of the MMD objective and the corresponding training procedure introduced in Section \ref{subsec:optimization_dynamics}. For notational simplicity throughout this theoretical discussion, we omit the boldface typography for state variables and vectors.

\subsection{Weak Convergence of Trained Distribution}

In this subsection, we review essential properties of reproducing kernel Hilbert spaces (RKHS) and examine the Maximum Mean Discrepancy (MMD) as a metric induced by such a space. Under the standard condition that the kernel $k$ is characteristic, the MMD metrizes the weak convergence of probability measures. Consequently, driving the MMD-based loss function $\mathcal{L}(\theta)$ to zero ensures that the trained distribution $p_\theta$ converges to the target $\pi$ in the weak topology.

The general theory of RKHS was developed by Aronszajn \cite{aronszajn1950theory}. Let $X$ be a locally compact Hausdorff space. The symmetric function $k:X\times X\to\mathbb{R}$ is integrally strictly positive definite if for arbitrary nonzero function $f:X\to\mathbb{R}$ it holds that
\begin{equation}
    \int_{X\times X}{k(x,y)f(x)f(y)}dxdy>0.
\end{equation}
Then there exists a unique Hilbert space $\mathcal{H}_k$ of functions on $X$ equipped with the inner product $\langle\cdot,\cdot\rangle_k$ satisfying:
\begin{itemize}
    \item $k_x\in\mathcal{H}_k$ for all $x\in X$, where $k_x(y):=K(x,y)$;
    \item $\mbox{span}\{k_x\}_{x\in X}$ is dense in $\mathcal{H}_k$;
    \item $f(x)=\langle f,k_x\rangle_k$ for all $f\in\mathcal{H}_k$.
\end{itemize}
And $\mathcal{H}_k$ is called the reproducing kernel Hilbert space with reproducing kernel $k$. When $k$ is bounded, we can define the kernel mean embedding $f_\mu(x):=\int_X{k(x,y)}d\mu(y)$ for probability measure $\mu$. For each measurable function $g$ on $X$, denote
\begin{equation*}
    \mu(g):=\int_X{g(x)}d\mu(x),
\end{equation*}
and it can be verified that the Pettis property holds,
\begin{equation}\label{eq:Pettis}
    \mu(g)=\langle f_\mu,g\rangle_k,\quad\forall\ g\in\mathcal{H}_k.
\end{equation}
In particular, for any two probability measures $\mu,\nu$, the MMD is defined as the RKHS norm of the difference between their kernel mean embeddings,
\begin{equation}\label{eq:kernel_norm}
    \|\mu-\nu\|_{MMD}:=\|f_\mu-f_\nu\|_k=\left(\int_{X\times X}{k(x,y)}(d\mu(x)-d\nu(x))(d\mu(y)-d\nu(y))\right)^\frac{1}{2}.
\end{equation}

The equivalence between MMD topology and weak topology was studied in \cite{simon2023metrizing}.
\begin{lemma}(Simon-Gabriel et al. 2023)\label{th:weak_convergence}
    Let $k$ be an integrally strictly positive definite kernel such that $\mathcal{H}_k\subset\mathcal{C}_0$, and let $\{\mu_\alpha\}$ and $\mu$ be Randon probability measures. If $k$ is continuous, then the following are equivalent:
    \begin{enumerate}[(i)]
        \item (convergence in strong RKHS topology) $\|\mu_\alpha-\mu\|_{MMD}\to0$;
        \item (convergence in weak RKHS topology) $\mu_\alpha(f)-\mu(f)\to0$ for all $f\in\mathcal{H}_k$;
        \item (convergence in weak-* topology) $\mu_\alpha(f)-\mu(f)\to0$ for all $f\in\mathcal{C}_0$;
        \item (convergence in weak topology) $\mu_\alpha(f)-\mu(f)\to0$ for all $f\in\mathcal{C}_b$.
    \end{enumerate}
\end{lemma}

Let $\mathcal{C}_0$ be the set of continuous functions on $X$ that vanish at infinity, and $\mathcal{C}_b$ the set of bounded continuous functions on $X$. By definition \eqref{eq:kernel_norm} and Pettis property \eqref{eq:Pettis}, $(i)\Rightarrow(ii)$ is clear. Since $\mathcal{H}_k$ is dense in $\mathcal{C}_0$ by \cite[Corollary 3, Theorem 8]{simon2018kernel}, and the set of Randon probability measures is equicontinuous by \cite[Theroem 33.2]{treves1967topological}, then \cite[Proposition 32,5]{treves1967topological} implies that $(ii)\Rightarrow(iii)$. At last, \cite[Corollary 2.4.3]{berg1984harmonic} yields $(iii)\Rightarrow(iv)$, and \cite[Theorem 2.3.3]{berg1984harmonic} yields $(iv)\Rightarrow(i)$.

Now we return to measures on $X:=\mathbb{R}^{d}$, and make use of Lemma \ref{th:weak_convergence} to derive the weak convergence of $p_\theta\to\pi$ as the loss $\mathcal{L}(\theta)$ tends to zero. Recall that by definition
\begin{equation*}
    \mathcal{L}(\theta)=\|p_\theta(x_1)p(x_1,x_2)-p_\theta(x_2)p(x_2,x_1)\|^2_{MMD},
\end{equation*}
where $p_\theta$ is the trained distribution, and $p(x_1,x_2)$ is the transition probability from state $x_1$ to state $x_2$. Then by Lemma \ref{th:weak_convergence}, when the trained distributions $\{p_\theta\}$ are Randon measures, $\mathcal{L}(\theta)\to0$ implies that $p_\theta(x_1)p(x_1,x_2)-p_\theta(x_2)p(x_2,x_1)\to 0$ in the weak topology. This does not directly imply the desired weak convergence $p_\theta(x_1)-\pi(x_1)\to0$. To bridge this gap, we require an additional technical lemma that connects the asymmetric difference involving the transition kernel $p(\cdot,\cdot)$ to a symmetric difference involving the target measure $\pi(\cdot)$. Recall that a sequence of probability measures $\mu_n$ is tight if for any $\varepsilon>0$, there is a compact set $A$ such that $\inf_{n}{\mu_n(A)}\geq1-\varepsilon$ \cite{durrett2019probability}. 

\begin{lemma}\label{th:target_convergence}
    Let the set of all possible distribution of the neural network $\{p_\theta\}$ be tight, and the continuous transition probability $p(x_1,x_2)$ strictly positive. Along a sequence $\{\theta_{\ell}\}$, if $p_\theta(x_1)p(x_1,x_2)-p_\theta(x_2)p(x_2,x_1)\to 0$ in weak topology, then $p_\theta(x_1)\pi(x_2)-\pi(x_1)p_\theta(x_2)\to0$ in weak topology.
\end{lemma}

\begin{proof}
    Let $f\in\mathcal{C}_b$. Then there exists $M>0$ such that $|f(x)|\leq M$ for all $x\in\mathbb{R}^{2d}$. And by tightness of $\{p_\theta\}_\theta$, for arbitrary $\varepsilon>0$, by tightness of $\{p_\theta\}$, there exists a compact set $A(\varepsilon)$ such that
    \begin{equation*}
        \int_{A^c}{\left|p_\theta(x_1)\pi(x_2)-\pi(x_1)p_\theta(x_2)\right|}dx_1dx_2<M^{-1}\varepsilon,
    \end{equation*}
    where $A^c$ is complement of $A$ in $\mathbb{R}^{2d}$. It follows that
    \begin{align*}
        &\left|\int_{A^c}{f(x_1,x_2)\left(p_\theta(x_1)\pi(x_2)-\pi(x_1)p_\theta(x_2)\right)}dx_1dx_2\right|\\
        &\qquad\leq\int_{A^c}{|f(x_1,x_2)|\left|p_\theta(x_1)\pi(x_2)-\pi(x_1)p_\theta(x_2)\right|}dx_1dx_2\\
        &\qquad\leq\int_{A^c}{M\left|p_\theta(x_1)\pi(x_2)-\pi(x_1)p_\theta(x_2)\right|}dx_1dx_2\\
        &\qquad\leq\varepsilon.
    \end{align*}
    Define the smooth truncation
    \begin{equation*}
        g(x_1,x_2):=f(x_1,x_2)\pi(x_2)p^{-1}(x_1,x_2)\chi_A(x_1,x_2),
    \end{equation*}
    where $\chi_A\in\mathcal{C}^{\infty}_c(\mathbb{R}^d)$ is a smooth cutoff function supported in a compact set, with $\chi_A=1$ on $A$ and $0\leq\chi_A\leq1$. Since $p(x_1,x_2)$ is strictly positive and continuous, $\pi(x_2)p^{-1}(x_1,x_2)$ is bounded by some constant $C(A)$ on the support of $\chi_A$. Then $g\in\mathcal{C}_b$, and for $\ell$ large enough the following inequality holds by the weak convergence of $p_\theta(x_1)p(x_1,x_2)-p_\theta(x_2)p(x_2,x_1)$,
    \begin{equation*}
        \left|\int_{\mathbb{R}^{2d}}{g(x_1,x_2)\left(p_{\theta_\ell}(x_1)p(x_1,x_2)-p_{\theta_\ell}(x_2)p(x_2,x_1)\right)}dx_1dx_2\right|<\varepsilon.
    \end{equation*}
    It follows by the detailed balance condition $\pi(x_1)p(x_1,x_2)=\pi(x_2)p(x_2,x_1)$,
    \begin{multline*}
  \left|\int_A{f(x_1,x_2)\left(p_{\theta_\ell}(x_1)\pi(x_2)-\pi(x_1)p_{\theta_\ell}(x_2)\right)}dx_1dx_2\right|\\
      =\left|\int_{\mathbb{R}^{2d}}{g(x_1,x_2)\left(p_{\theta_\ell}(x_1)p(x_1,x_2)-p_{\theta_\ell}(x_2)p(x_2,x_1)\right)}dx_1dx_2\right|<\varepsilon.
    \end{multline*}
    which then verifies the lemma.
\end{proof}

Note that $p_\theta(x_1)-\pi(x_1)$ is the marginal distribution of $p_\theta(x_1)\pi(x_2)-\pi(x_1)p_\theta(x_2)$, the following Theorem \ref{th:convergence} is a straight result of Lemma \ref{th:weak_convergence} and Lemma \ref{th:target_convergence}.

\begin{theorem}\label{th:convergence}
    Let $k$ be a continuous integrally strictly positive definite kernel such that $\mathcal{H}_k\subset\mathcal{C}_0$. Assume the set of all possible distribution of the neural network $\{p_\theta\}$ is a tight set of probability measures, and the continuous transition probability $p(x_1,x_2)$ is strictly positive. Then $p_\theta$ converges to $\pi$ in weak topology as $\mathcal{L}(\theta)\to 0$.
\end{theorem}

By \cite[Lemma 8]{simon2023metrizing}, $\mathcal{H}_k\subset\mathcal{C}_0$ if and only if $k$ is bounded and $k(\cdot,x)\in\mathcal{C}_0$ for all $x\in X$. Then the condition $\mathcal{H}_k\subset\mathcal{C}_0$ is satisfied by many standard kernels: Gaussian, Laplacian, Matern and inverse multi-quadratic kernels.

\subsection{Approximation and training}
In this subsection, we make some discussions on the gradient of the loss function to demonstrate validity of our algorithm. Theoretically, the loss function is defined as
\begin{equation*}
    \mathcal{L}(\theta)=\iint_{\mathbb{R}^{2d}\times \mathbb{R}^{2d}}{k(\vec{x},\vec{y})(\mu_1(\vec{x})-\mu_2(\vec{x}))(\mu_1(\vec{y})-\mu_2(\vec{y}))d\vec{x}}d\vec{y},
\end{equation*}
where $\vec{x}:=(x_1,x_2)\in\mathbb{R}^{2d}$ and
\begin{equation*}
\mu_1(\vec{x}):=p_\theta(x_1)p(x_1,x_2),\qquad\mu_2(\vec{x}):=p_\theta(x_2)p(x_2,x_1).
\end{equation*}
Here both $\mu_1$ and $\mu_2$ depend implicitly on the parameter $\theta$; for brevity, this dependence is omitted in the notation below.

The first attempt is then to utilize the gradient of this loss to train the model. As direct evaluation of the continuous gradient is typically intractable, one may have two options. The first is to compute the continuous gradient first and then approximate using the Monte Carlo approach, while the second one is to approximate the loss using the Monte Carlo approach first and then take the gradient.

Here, we explain a little bit about the Monte Carlo approximation. In the generative framework, let $G_\theta:\mathbb{R}^d\to\mathbb{R}^d$ is a differentiable transport map from a base distribution $x\sim p_0$ to the target distribution $G_\theta(x)\sim p_\theta$. During training, we generate samples $\{X_i\}^N_{i=1}$ from $p_0$ via $X_i=G_\theta(Z_i)$ with $Z_i\sim p_0$. For each $X_i$, we then generate $Y_i$ by simulating one step of a Markov process with initial state $X_i$ and transition kernel $p(\cdot,\cdot)$. This yields paired samples $\{(X_i,Y_i)\}^N_{i=1}$. We then define the empirical distribution
\begin{equation*}
    \mu^N_1(\vec{x}):=\frac{1}{N}\sum^N_{i=1}{\delta_{(X_i,Y_i)}(\vec{x})},\qquad\mu^N_2(\vec{x}):=\frac{1}{N}\sum^N_{i=1}{\delta_{(Y_i,X_i)}(\vec{x})},
\end{equation*}
which approximate $\mu_1$ and $\mu_2$ respectively.

We see the difference: in the continuous setting, the dependence on $\theta$ is in the $p_{\theta}$, while the dependence is moved to the samples 
after the Monte Carlo approximation.

\subsubsection{Issues for the full gradient}

 Next, we discuss why using the full gradient of the loss could be troublesome for our use. To derive the continuous gradient, we assume that first that the state space is continuous and the generator $G_\theta$ satisfies the ordinary differential equation
\begin{equation}\label{eq:generatorode}
    \frac{d}{d\theta}G_\theta(x)=v(\theta,G_\theta(x)).
\end{equation}
Note that such velocity field $v$ exists in general if $G_{\theta}$ has certain smoothness. The density $p_\theta$ then evolves according to the transport equation
\begin{equation*}
    \partial_\theta p_\theta(x)=-\nabla\cdot(v(\theta,x)p_\theta(x)).
\end{equation*}
Differentiating the loss $\mathcal{L}(\theta)=\|\mu_1-\mu_2\|^2_{MMD}$ with respect to $\theta$ gives
\begin{equation*}
    \frac{d}{d\theta}\mathcal{L}(\theta)=2\int_{\mathbb{R}^{2d}}{\int_{\mathbb{R}^{2d}}{k(\vec{x},\vec{y})(\partial_\theta p_\theta(x_1)p(x_1,x_2)-\partial_\theta p_\theta(x_2)p(x_2,x_1))(\mu_1(\vec{y})-\mu_2(\vec{y}))}d\vec{x}}d\vec{y},
\end{equation*}
where $\vec{x}=(x_1,x_2)$ and $\vec{y}=(y_1,y_2)$. To obtain an expression suitable for Monte Carlo estimation, we perform integration by parts. Using the transport equation and assuming sufficient decay at infinity, we obtain
\begin{multline}\label{eq:fullgrad}
    \frac{d}{d\theta}\mathcal{L}(\theta)
    =2\int_{\mathbb{R}^{2d}}{\int_{\mathbb{R}^{2d}}{v(\theta,x_1)\cdot\nabla_{x_1}\left[\ln{k(\vec{x},\vec{y})}+\ln{p(x_1,x_2)}\right]k(\vec{x},\vec{y})\mu_1(\vec{x})(\mu_1(\vec{y})-\mu_2(\vec{y}))}d\vec{x}}d\vec{y}\\
    -2\int_{\mathbb{R}^{2d}}{\int_{\mathbb{R}^{2d}}{v(\theta,x_2)\cdot\nabla_{x_2}\left[\ln{k(\vec{x},\vec{y})}+\ln{p(x_2,x_1)}\right]k(\vec{x},\vec{y})\mu_2(\vec{x})(\mu_1(\vec{y})-\mu_2(\vec{y}))}d\vec{x}}d\vec{y}.
\end{multline}
Here $\nabla_{x_1}\ln{p(x_1,x_2)}$ denotes the gradient with respect to the first argument of the transition kernel $p(\cdot,\cdot)$; similarly, $\nabla_{x_2}\ln{p(x_2,x_1)}$ is the gradient with respect to the first argument evaluated at $(x_1,x_2)$.  When $p_\theta=\pi$ satisfies the detailed balance condition $\mu_1(\vec{y})-\mu_2(\vec{y})=0$ for all $\vec{y}$, the gradient vanishes. This confirms that the target distribution $\pi$ is a stationary point of the loss landscape, as expected.

Approximation for continuous gradient using the Monte Carlo approximation is feasible for continuous state problems: $v(\theta, x)$ is easy to approximate using $\partial_{\theta}G_{\theta}$, while $\nabla_{x_i}\ln{p(x_1,x_2)}$ can be computed in principle. The problem is that the formula could be complicated and different formulas would be used if there are discrete components. In fact, in the case that $x_i$ has discrete components which is our focus in this work, the ODE \eqref{eq:generatorode} should be replaced by some difference equation. Some analogue could be derived but one must discuss different cases, which could be much  involved. Another problem for this approach is that one cannot use the auto-differentiation using the backpropagation in Python conveniently.
To summarize, computing the continuous gradient first and then approximate could be doable but there is much difficulty involved.

Applying the Monte Carlo approximation first and then taking the gradient is preferred in practice as one could make use the auto-differentiation using the backpropagation in Python. However, this approach is not feasible for our problem. In fact, the loss using the samples is given by
\begin{align*}
    \mathcal{L}^N(\theta)
    &=\int_{\mathbb{R}^{2d}}{\int_{\mathbb{R}^{2d}}{k(\vec{x},\vec{y})(\mu^N_1(\vec{x})-\mu^N_2(\vec{x}))(\mu^N_1(\vec{y})-\mu^N_2(\vec{y}))}d\vec{x}}d\vec{y}\\
    &=\frac{1}{N^2}\sum^N_{i,j=1}{\left[k((X_i,Y_i),(X_j,Y_j))-k((X_i,Y_i),(Y_j,X_j))\right]}\\
    &\quad-\frac{1}{N^2}\sum^N_{i,j=1}{\left[k((Y_i,X_i),(X_j,Y_j))-k((Y_i,X_i),(Y_j,X_j))\right]}.
\end{align*}

Here, the $\theta$ dependence is moved to the samples, which is almost impossible to compute the gradient because there is sampling procedure involved like sampling from the proposal distribution and the acceptance-rejection procedure.


\subsubsection{Discussion on the surrogate gradient}

As mentioned, the surrogate gradient in \eqref{eq:grad_stop_kernel} is obtained by approximating the loss using Monte Carlo first and then detaching $\mathbf{s}'$ from the computation graph. In other words, the dependence of $\theta$ of $Y_i$ is ignored.
The reason has been explained above due to the infeasibility of taking $\theta$ derivative for $Y_i$. Training using such a gradient surrogate is certainly convenient, as one only needs to take the derivative of the generator $G_{\theta}$. For simplicity, we assume in this subsection that the kernel $k$ is symmetric: $k(\vec{x},\vec{y})=k(\vec{y},\vec{x})$, and the analysis can be extended to the non-symmetric case without essential difficulty.

With the Monte Carlo approximation, \eqref{eq:grad_stop_kernel} is given by
\begin{equation}
    \begin{aligned}
        \hat{g}_N(\theta)
        &=\frac{2}{N^2}\sum^N_{i,j=1}{\frac{d}{d\theta}G_\theta(Z_i)\cdot\left[\nabla_{x_1}k((X_i,Y_i),(X_j,Y_j))-\nabla_{x_1}k((X_i,Y_i),(Y_j,X_j))\right]}\\
        &\quad-\frac{2}{N^2}\sum^N_{i,j=1}{\frac{d}{d\theta}G_\theta(Z_i)\cdot\left[\nabla_{x_2}k((Y_i,X_i),(X_j,Y_j))-\nabla_{x_2}k((Y_i,X_i),(Y_j,X_j))\right]}.
    \end{aligned}
    \label{eq:surrogate_grad}
\end{equation}
Recall that by definition $X_i=G_\theta(Z_i)$ with $Z_i\sim p_0$.

As explained in Appendix \ref{app:Nlimitofsurrogate}, this is equal to approximating the loss using Monte Carlo first and then taking the surrogate gradient by the above strategy. The $N\to\infty$ limit of the discrete surrogate gradient above, or \eqref{eq:grad_stop_kernel}, is explicitly given by the following formula
\begin{equation}
\begin{split}
    \hat{g}_{\infty}(\theta)
    &=2\int_{\mathbb{R}^{2d}}{\int_{\mathbb{R}^{2d}}{v(\theta,x_1)\cdot\nabla_{x_1}[k(\vec{x},\vec{y})]\mu_1(\vec{x})(\mu_1(\vec{y})-\mu_2(\vec{y}))}d\vec{x}}d\vec{y}\\
    &\quad-2\int_{\mathbb{R}^{2d}}{\int_{\mathbb{R}^{2d}}{v(\theta,x_2)\cdot\nabla_{x_2}[k(\vec{x},\vec{y})]\mu_2(\vec{x})(\mu_1(\vec{y})-\mu_2(\vec{y}))}d\vec{x}}d\vec{y}.
\end{split}
\label{eq:inf_surrogate_grad}
\end{equation}
Compared to \eqref{eq:fullgrad}, one finds that 
this simplification has ignored $\nabla_{x_i}\ln{p(x_1,x_2)}$ in the formula. 
The following property still holds for this continuous surrogate gradient.
\begin{proposition}
The continuous surrogate gradient $\hat{g}_{\infty}=0$ when $p_{\theta}=\pi$.
\end{proposition}
\noindent This is obvious as $\mu_1=\mu_2$ in this case so we skip the proof. This property ensures that the equilibrium point of the new dynamics is the same as the original one.

A natural question arises: does omitting the term $\nabla_{x_1}\ln{p(x_1,x_2)}$ in the optimization of $\mathcal{L}(\theta)$ preserve a descent direction?
In the classical Metropolis-Hastings Markov chain, given a proposal kernel $q(x_1,x_2)$ from state $x_1$ to $x_2$ and a target stationary distribution $\pi(x)$, the transition probability is defined as
\begin{equation}\label{eq:transition_probability}
    p(x_1,x_2)=\min\left\{q(x_1,x_2),\frac{\pi(x_2)}{\pi(x_1)}q(x_2,x_1)\right\}.
\end{equation}
Differentiating the logarithm yields the piecewise expression
\begin{equation*}
    \nabla_{x_1}\ln{p(x_1,x_2)}=\left\{\begin{aligned}
        &\nabla_{x_1}q(x_1,x_2),&\frac{\pi(x_1)q(x_1,x_2)}{\pi(x_2)q(x_2,x_1)}\leq1,\\
        &\pi(x_2)\nabla_{x_1}\left(\frac{q(x_2,x_1)}{\pi(x_1)}\right),&\frac{\pi(x_2)q(x_2,x_1)}{\pi(x_1)q(x_1,x_2)}<1.
    \end{aligned}\right.
\end{equation*}
This expression simplifies considerably when the proposal kernel $q(\cdot,\cdot)$ is chosen to be symmetric. In practice, both the kernel $k(\cdot,\cdot)$ appearing in the MMD and the proposal $q(\cdot,\cdot)$ are taken to be symmetric. 
Under this symmetry, the discussion is reduced to that of $\nabla_{x_1}$ only ($\nabla_{x_2}$ is similar). 

For the surrogate gradient to be a good descent direction, it is desired that
\begin{equation*}
    \left|\nabla_{x_1}q(x_1,x_2)\right|\ll\left|\nabla_{x_1}k(\vec{x},\vec{y})\right|,\qquad\left|\nabla_{x_1}\left(\frac{q(x_2,x_1)}{\pi(x_1)}\right)\right|\ll\left|\frac{\nabla_{x_1}k(\vec{x},\vec{y})}{\pi(x_1)}\right|.
\end{equation*}
Of course, this inequality cannot be proved as the right hand sides of both inequalities contain $\vec{y}$. However, typically, if the proposal gives a sample $x_2$ close to $x_1$, we may expect $\nabla_{x_1}q(x_1, x_2)$ to be small while the right hand side could be large if the typical $\vec{y}$ is moderately away from $\vec{x}$. This intuition suggests that the surrogate gradient could be a good approximation, and could yield a valid descent direction.

Of course, even though it is not strict descent direction, if the dynamics could eventually converge to its stationary state, it is still good as the stationary point is the same. As another remark, even if we cannot guarantee that the convergence to the desired stationary solution. The solution found using the surrogate gradient could be a good initial guess for another optimization method that can ensure in theory the convergence to the minimizer (like the gradient-free methods). Nevertheless, as we find the numerical simulation below, using the surrogate gradient itself is already enough for training the models in all our examples.
A rigorous analysis of the resulting optimization dynamics is nontrivial and is left for future work.

%% file: section/numerical_experiments.tex
\section{Numerical Experiments}
\label{sec:experiments}

\subsection{Continuous Validation: 2D Gaussian Mixture}
\label{subsec:continuous_exp}

To validate the effectiveness of our framework on continuous multimodal distributions, we first apply it to a two-dimensional Gaussian mixture model (GMM). This strictly continuous setting serves as a rigorous sanity check to confirm that our reversibility-based objective can correctly recover complex energy landscapes and transition across low-density regions before moving to discrete domains.

\paragraph{Experimental Setup.}
The target distribution $\pi(\mathbf{x})$ is defined as a mixture of two Gaussian components:
\begin{equation}
    \pi(\mathbf{x}) = \pi_1 \mathcal{N}(\mathbf{x} \mid \boldsymbol{\mu}_1, \boldsymbol{\Sigma}_1) + \pi_2 \mathcal{N}(\mathbf{x} \mid \boldsymbol{\mu}_2, \boldsymbol{\Sigma}_2),
\end{equation}
with unequal mixing coefficients $\pi_1 = 0.6$ and $\pi_2 = 0.4$, located at $\boldsymbol{\mu}_1 = [1, 1]^\top$ and $\boldsymbol{\mu}_2 = [-1, -1]^\top$, respectively. The covariance matrices are set to introduce distinct correlations: $\boldsymbol{\Sigma}_1$ has positive off-diagonal elements ($0.2$), whereas $\boldsymbol{\Sigma}_2$ has negative ones ($-0.2$), alongside a shared variance of $0.5$.

For the generator, we employ a RealNVP architecture \cite{dinh2016density} consisting of 8 affine coupling layers, where the scale and translation functions are parameterized by 2-layer MLPs with 64 hidden units and LeakyReLU activations. Crucially, while our proposed MMD training objective is strictly Jacobian-free, utilizing an invertible flow here allows us to compute the exact log-likelihoods required for rigorous downstream quantitative evaluation (e.g., KL divergence).

To ensure robust gradient flow across the disconnected modes and prevent mode collapse, the MMD objective utilizes a multi-scale kernel that combines Gaussian RBFs (bandwidths $\sigma \in \{0.1, 0.5, 1.0, 2.0, 5.0\}$) and an Inverse Multiquadric (IMQ) kernel (with parameters $\beta=0.5$ and $c_{\text{imq}}=1.4$). This combination provides the heavy-tailed gradients crucial for learning long-range dependencies \cite{binkowski2018demystifying}. 

Additionally, to confine generated samples within a valid domain and ensure stable early exploration, we add a soft boundary penalty to the overall loss \cite{cai2024weak}:
\begin{equation}
    L_b = \frac{1}{N} \sum_{i=1}^N \text{Sigmoid}\Big( c\cdot\big(\|G_{\theta}(\mathbf{z}_{i})-x_0\|_{2}^{2} - r^{2}\big)\Big).
    \label{eq:boundary_penalty}
\end{equation}

\paragraph{MCMC Transition Kernel.}
To construct the forward and backward joint distributions required by the reversibility-based MMD objective, we define a continuous Metropolis-Hastings transition kernel $p(\mathbf{x},\mathbf{x}')$. Specifically, we employ an isotropic Gaussian random walk proposal, $q(\mathbf{x},\mathbf{x}') = \mathcal{N}(\mathbf{x}' \,|\, \mathbf{x}, \sigma_{\text{prop}}^2 \mathbf{I})$, where the step size is set to $\sigma_{\text{prop}} = [0.1]$. During training, we apply $m = [3]$ sequential MCMC steps to the generated samples to evaluate the reversibility violation. This simple yet effective local proposal provides sufficient exploration in the continuous landscape while maintaining detailed balance.

The model is optimized using the AdamW optimizer with a batch size of $2048$. We set an initial learning rate of $10^{-4}$ which is decayed by a factor of $0.71$ at $20{,}000$, $50{,}000$  and $100{,}000$ iterations to ensure stable convergence.

\paragraph{Results.}
Figure~\ref{fig:mix_gaussian} presents the qualitative sampling results. 
As shown in the \textbf{top row}, the trained generator successfully recovers the asymmetric bimodal structure of the target density, accurately capturing the relative probability mass of the two modes ($0.6$ vs. $0.4$) without suffering from mode dropping.
This alignment is further corroborated by the empirical marginal distributions shown in the \textbf{bottom row}, where the generated profiles closely match the ground truth.

To quantitatively assess convergence speed and stability, Figure~\ref{fig:quant_results} plots the training objective (MMD loss) and validation metrics ($L_2$ error and KL divergence) against both training iterations and wall-clock time.
As shown in Figure~\ref{fig:quant_results}(a), the MMD loss decreases rapidly within the initial phase, followed by a steady decline in both $L_2$ error and KL divergence. This confirms that penalizing the reversibility violation effectively drives the generator toward the global equilibrium.
Furthermore, Figure~\ref{fig:quant_results}(b) demonstrates that high-fidelity sampling is achieved with minimal computational overhead. Ultimately, the model reaches an exceptionally low $L_2$ density error of $0.0483$, robustly validating the accuracy of the proposed framework.

\begin{figure}[htbp]
    \centering
    \includegraphics[width=0.8\linewidth]{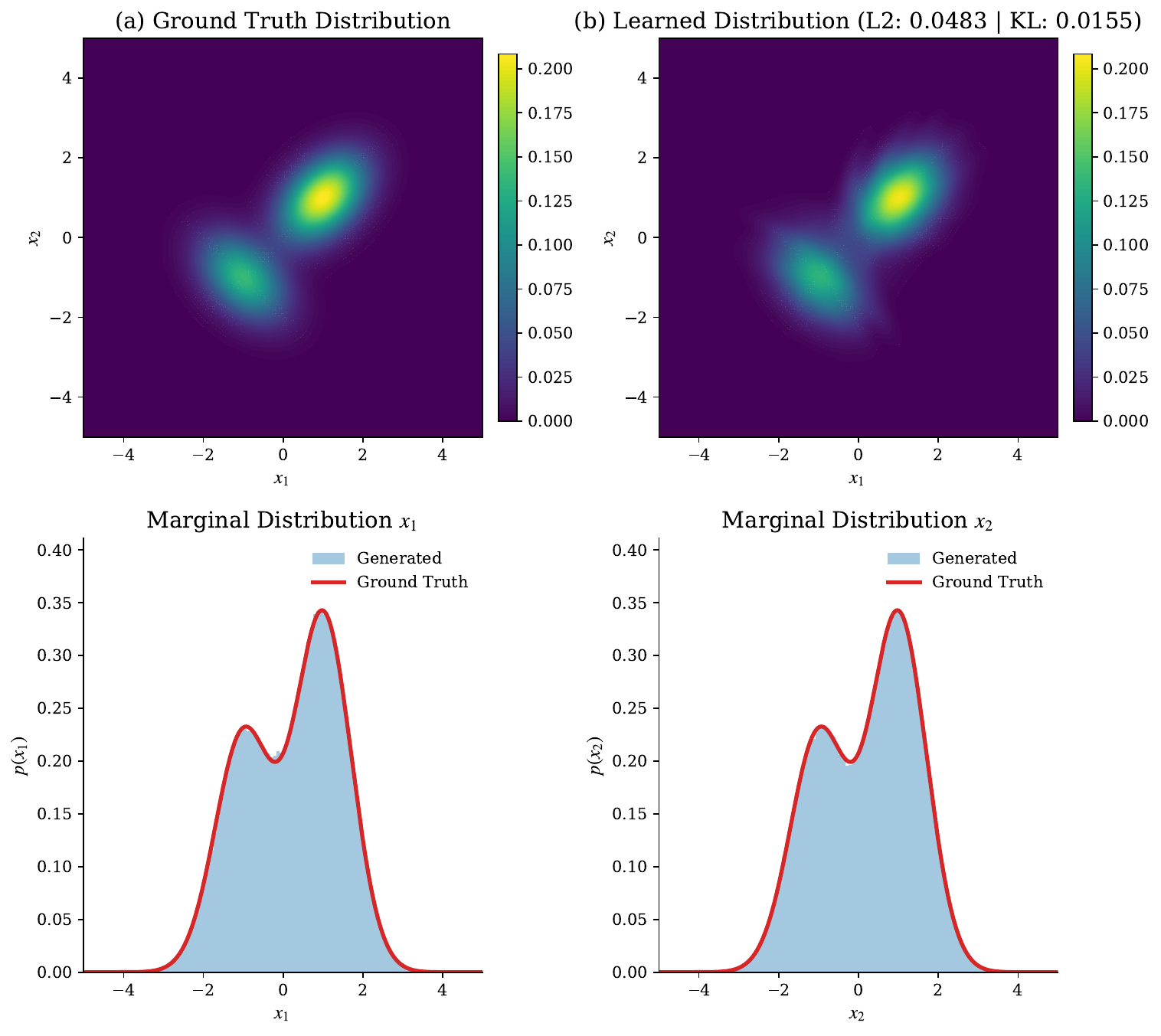} 
    \caption{Qualitative results on the 2D Gaussian mixture. \textbf{Top-left:} Ground truth target density. \textbf{Top-right:} Learned density by our model. \textbf{Bottom row:} Marginal distributions along the first (left) and second (right) dimensions. The learned density accurately captures both modes and their relative weights.}
    \label{fig:mix_gaussian}
\end{figure}

\begin{figure}[htbp]
    \centering
    \includegraphics[width=0.8\linewidth]{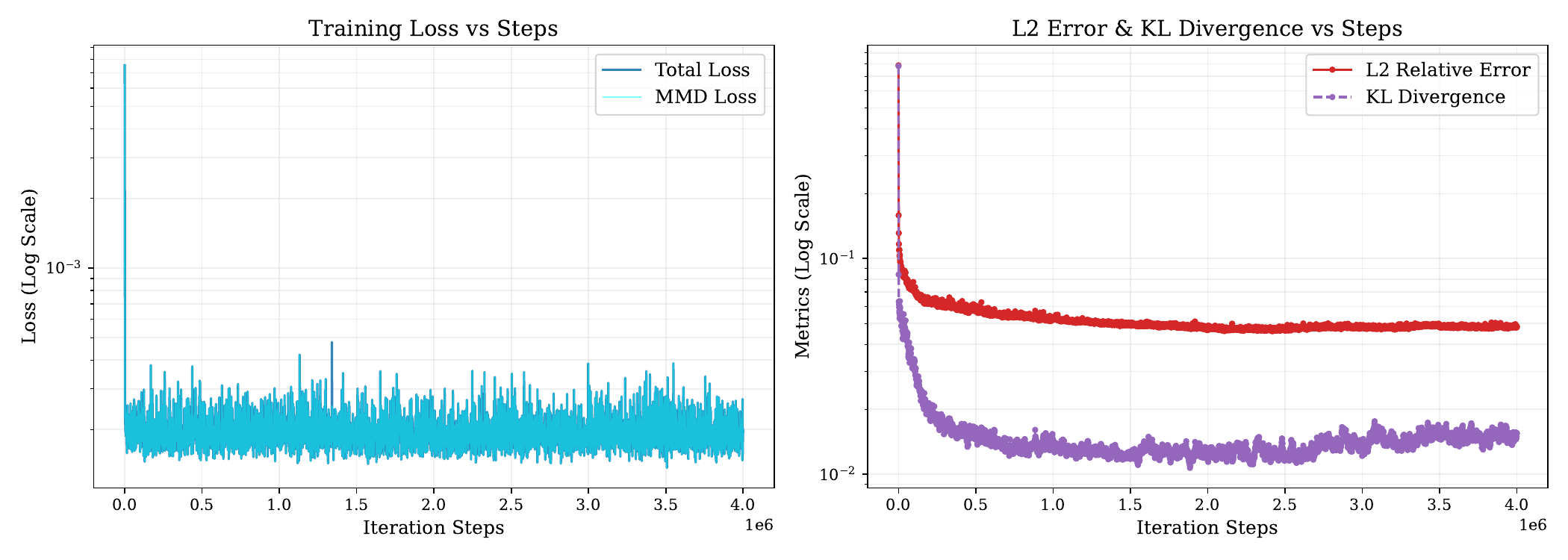}
    \caption{Training curves}
    \label{fig:quant_results}
\end{figure}

\subsection{Hybrid System: Balanced Double Well Potential}
\label{subsec:hybrid_exp}

To demonstrate the versatility of our framework across mixed continuous-discrete domains, we evaluate it on a hybrid system where the state is denoted by $\mathbf{s} = (x, k)$. Here, $x \in \mathbb{R}$ is a continuous coordinate and $k \in \{0, 1, \dots, M-1\}$ is a discrete mode index.

\paragraph{System Definition.}
The energy landscape is defined by a conditional quartic potential, where the discrete mode $k \in \{0, 1, 2\}$ determines the geometric properties of the potential well acting on $x$:
\begin{equation}
    E(x, k) = (x^2 - \mu)^2.
\end{equation}
To strictly assess the generator's ability to cross massive energy barriers, we set the well positions to extreme values ${\mu} = [1.0, 9.0, 25.0]$, creating a highly heterogeneous mixture of spatial extents. The target joint probability is explicitly equalized across discrete modes, $p(x, k) = \frac{1}{M} e^{-E(x, k)} / Z_k$, requiring the generator to overcome the severe energy barriers between distinct $k$-modes while correctly sampling the local fluctuations of $x$.

\paragraph{Split-Head Generator Architecture.}
To capture the strict statistical coupling between $x$ and $k$, we propose a \textbf{split-head} architecture. Instead of independent generators, a shared 3-layer MLP backbone (with 128 hidden units per layer) first maps standard Gaussian noise $\mathbf{z} \in \mathbb{R}^{32}$ to a common latent representation $\mathbf{h}$. This representation is then bifurcated into continuous and discrete heads:
\begin{equation}
    x = \mathbf{W}_x \mathbf{h} + \mathbf{b}_x, \quad \text{and} \quad k = \text{Sample}(\text{Softmax}(\mathbf{W}_k \mathbf{h} + \mathbf{b}_k)).
\end{equation}
Consistent with our approach for strictly discrete systems, the discrete head utilizes the Straight-Through Estimator (STE) \cite{bengio2013estimating} solely during the backward pass to enable end-to-end differentiability. This shared design ensures that the generated continuous coordinate $x$ is physically consistent with the sampled mode $k$.

\paragraph{Product Kernel and Hybrid MCMC Transition.}
Standard kernels are ill-defined for mixed spaces. To compute the MMD effectively, we define a Product Kernel that factorizes the similarity:
\begin{equation}
    k_{\text{MMD}}\big((x, k), (x', k')\big) = k_{\text{RBF}}(x, x') \cdot \mathbb{I}(k = k'),
    \label{eq:product_kernel}
\end{equation}
where $\mathbb{I}$ is the indicator function. This formulation dictates that two samples are considered proximal only if they exactly match in the discrete dimension and are close in the continuous space. 

To construct the joint distributions for the MMD objective, we employ a two-phase hybrid Metropolis-Hastings kernel $p\big((x, k) ,(x', k')\big)$ to efficiently navigate the heterogeneous phase space:
\begin{enumerate}
    \item \textbf{Intra-mode Exploration (Fixed $k$):} We use a mixture proposal combining a standard Gaussian random walk $x' = x + \epsilon$ (where $\epsilon \sim \mathcal{N}(0, 0.5^2)$) with probability $0.9$, and a symmetry-breaking flip $x' = -x + \epsilon$ with probability $0.1$ to efficiently cross the local barrier at $x=0$.
    \item \textbf{Cross-mode Teleportation (Joint Update):} To ensure ergodicity across isolated wells, we uniformly propose a transition to a different mode $k' \neq k$. To accommodate the varying spatial extents, the continuous coordinate is deterministically mapped via $x' = x \sqrt{\mu_{k'} / \mu_k}$. To rigorously maintain detailed balance under this spatial stretching, the MH acceptance probability is simply adjusted by the scaling factor $\sqrt{\mu_{k'}/\mu_k}$.
\end{enumerate}

\paragraph{Optimization details.}
The MMD objective is evaluated using joint distributions generated via $3$ sequential steps of this exact hybrid kernel. The network is optimized using the AdamW optimizer with a batch size of $2048$. To ensure stable convergence across the extreme energy landscape, we apply gradient clipping with a maximum norm of $1.0$. The generator is trained for a total of $100{,}000$ iterations, employing a cosine annealing learning rate schedule that decays from an initial value of $5 \times 10^{-4}$ down to $10^{-6}$.

\begin{figure}[H]
    \centering
    \begin{subfigure}[b]{0.64\linewidth}  
        \centering
        \includegraphics[width=1.0\linewidth]{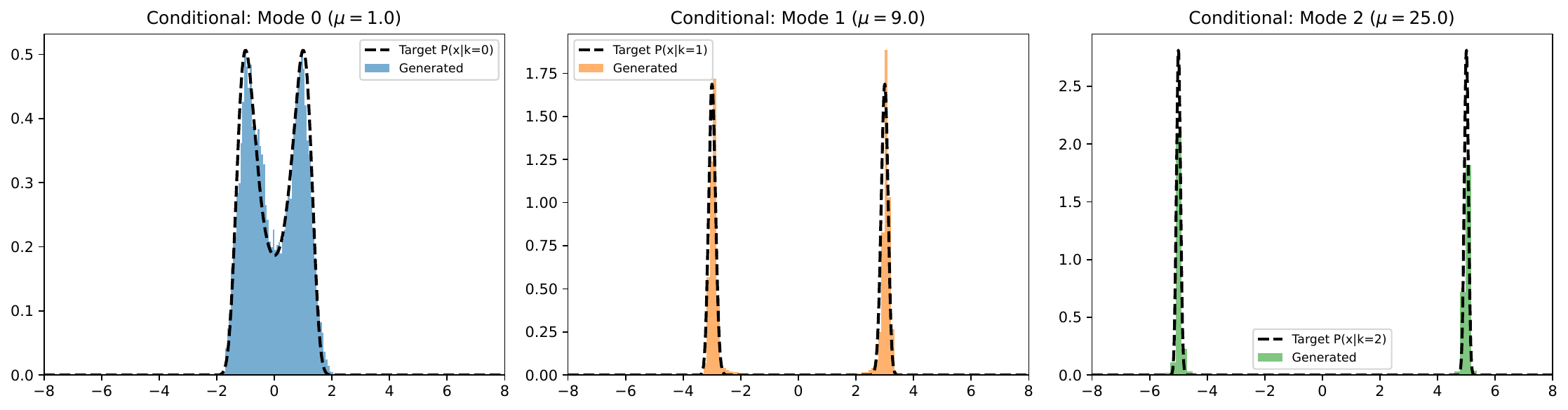}
        \caption{Conditional distribution}
        \label{fig:cond_dist}
    \end{subfigure}
    \begin{subfigure}[b]{0.35\linewidth}
        \centering
        \includegraphics[width=0.65\linewidth]{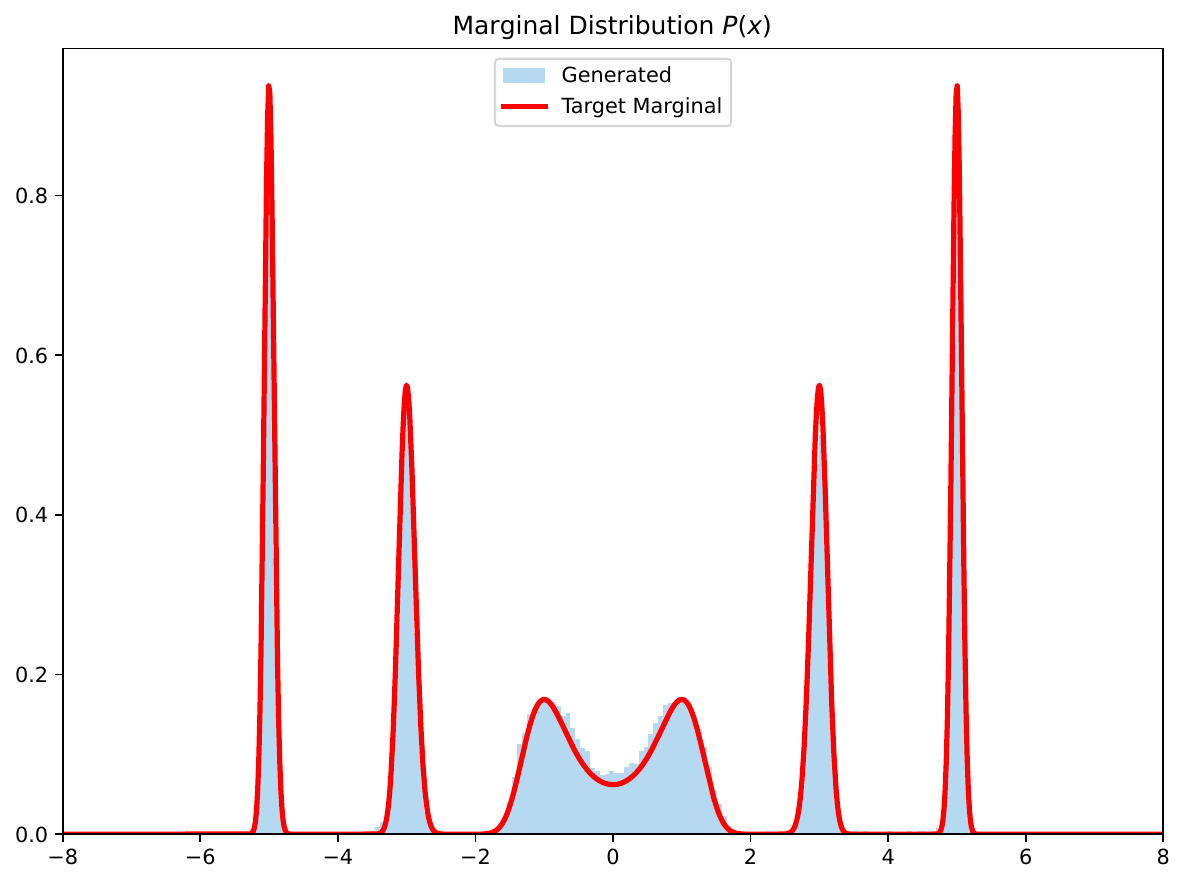}
        \caption{Marginal distribution}
        \label{fig:marginal_dist}
    \end{subfigure}
    \caption{Visual evaluation of the generative fidelity on the balanced double-well system. \textbf{(a)} Conditional distributions of the continuous coordinate $x$ given the discrete mode $k$. The generated samples perfectly align with the theoretical probability density functions across all three modes ($\mu \in \{1.0, 9.0, 25.0\}$), demonstrating the model's precise capture of local well geometries and thermal fluctuations. \textbf{(b)} The marginal distribution of $x$ aggregated over all modes. The symmetric and globally matched profile confirms the absence of mode collapse and validates the balanced exploration of the phase space.}
    \label{fig:combined_distribution}
\end{figure}

\begin{table}[H]
    \centering
    \caption{Comprehensive quantitative evaluation on the Balanced Double-Well system ($\mu \in \{1.0, 9.0, 25.0\}$). We report errors across discrete mode selection, conditional continuous moments, and global distribution distances.}
    \label{tab:balanced_double_well_ultimate}
    \begin{tabular}{@{}lc@{}}
        \toprule
        \textbf{Metric} & \textbf{Generated (Ours)} \\ 
        \midrule
        \multicolumn{2}{l}{\textit{Discrete Mode Selection}} \\
        \quad Mode L1 Error $\sum |P_{gen}(k) - P_{true}(k)|$   & $0.0382$ \\
        \midrule
        \multicolumn{2}{l}{\textit{Continuous Fidelity (Conditional on modes)}} \\
        \quad Mean Conditional $W_1$                            & $0.0399$ \\
        \midrule
        \multicolumn{2}{l}{\textit{Global \& Joint Distribution}} \\
        \quad Marginal $W_1$ w.r.t $x$                          & $0.0886$ \\
        \quad Joint MMD                       & $3.9324\times 10^{-3}$ \\
        \bottomrule
    \end{tabular}
\end{table}

\subsection{Discrete System: 2D Ising Model}
\label{subsec:ising_exp}

To demonstrate the capability of our framework in strictly discrete domains where Jacobian-based methods are inapplicable, we investigate the 2D Ising model \cite{ising1925beitrag, onsager1944crystal}. This system serves as a standard benchmark for discrete generative modeling due to its complex combinatorial state space and well-characterized phase transition behavior.

\paragraph{Physical Setup.}
We consider a square lattice $\Lambda$ of size $L \times L$ with periodic boundary conditions. The system state is defined by discrete spins $\mathbf{s} \in \{-1, +1\}^{N}$, where $N=L^2$. The Hamiltonian is given by:
\begin{equation}
    E(\mathbf{s}) = -J \sum_{\langle i, j \rangle} s_i s_j - h \sum_{i \in \Lambda} s_i,
    \label{eq:ising_hamiltonian}
\end{equation}
where $\langle i, j \rangle$ denotes nearest-neighbor pairs, $J>0$ is the ferromagnetic coupling constant, and $h$ is the external magnetic field. In our experiments, we set $J=1$ and strictly assume zero external field ($h=0$) to preserve the fundamental $\mathbb{Z}_2$ spin-flip symmetry, varying only the inverse temperature $\beta$. 
We deliberately focus on a small-scale lattice ($L=3$, $N=9$) where the exact partition function $Z = \sum_\mathbf{s} e^{-\beta E(\mathbf{s})}$ can be fully enumerated. This allows us to compute the ground truth probabilities analytically, providing a rigorous and absolute baseline for evaluating generative accuracy.

\paragraph{Generator with Straight-Through Estimator.}
The intrinsically discrete nature of spins fundamentally breaks the continuous computation graph, posing a formidable challenge for gradient-based optimization. To address this, we parameterize the generator using a 3-layer Multi-Layer Perceptron (MLP) with 256 hidden units and LeakyReLU activations. The MLP maps a standard isotropic Gaussian noise vector $\mathbf{z} \sim \mathcal{N}(\mathbf{0}, \mathbf{I})$ ($d=32$) to real-valued unnormalized logits. To enable end-to-end differentiability while strictly enforcing the discrete domain, we utilize the Straight-Through Estimator (STE) \cite{bengio2013estimating}. 
Specifically, the forward pass applies the $\text{sign}(\cdot)$ function to cast the logits into physical states $\mathbf{s} \in \{-1, +1\}^N$, while the backward pass approximates the gradients using the derivative of the $\tanh(\cdot)$ function. 
Crucially, because our reversibility-based MMD objective evaluates samples directly through energy differences and does not require differentiating the target Boltzmann density, the STE is exclusively needed to route gradients back to the generator's internal parameters. This avoids the severe training instabilities often encountered in score-based or reinforcement learning approaches on discrete spaces.

\paragraph{MCMC Transition Kernel.}
To compute the reversibility-based MMD objective, we must define a valid physical transition kernel $p(\mathbf{s},\mathbf{s}')$ to generate the forward and backward joint pairs. We utilize a Metropolis-Hastings sampler equipped with symmetric proposal distributions, and empirically validate two effective proposal strategies: (1) a \textit{mixture proposal} that performs a random single-spin flip with high probability and a global all-spin flip with low probability, which is particularly advantageous for efficiently traversing the global $\mathbb{Z}_2$ symmetry barriers in the low-temperature regime; and (2) a \textit{multi-spin proposal} that uniformly samples an integer $n \in [1, N]$ and simultaneously flips $n$ randomly chosen spins. Both proposal strategies inherently satisfy detailed balance and provide sufficient ergodicity to reliably guide the generator toward the target Boltzmann distribution.

\paragraph{Optimization and Results.}
The model is trained using the AdamW optimizer with a batch size of $2048$. Given the varying thermodynamic difficulty across different temperature regimes, we employ an initial learning rate of $10^{-3}$ and apply a step-decay schedule, where the total number of iterations and the specific decay milestones are adaptively tuned based on the inverse temperature $\beta$.
We comprehensively evaluate the model across two distinct thermodynamic regimes: the high-temperature disordered phase ($\beta=0.2$) and the low-temperature ordered phase ($\beta=0.5$). For each regime, we generate $200{,}000$ independent samples for rigorous quantitative assessment. The generative fidelity is visually and quantitatively compared against the exact Boltzmann distribution in Figures~\ref{fig:ising_beta02} and \ref{fig:ising_beta05}, and Table~\ref{tab:ising_merged_results}.

\begin{itemize}
    \item \textbf{High-Temperature Regime ($\beta=0.2$):} As shown in Figure~\ref{fig:ising_beta02}, strong thermal fluctuations dominate the system. The magnetization distribution (top-left) exhibits a unimodal profile centered at zero, representing a purely disordered state. The generator accurately captures this behavior, perfectly matching the ground truth energy spectrum (top-right) and correctly assigning precise probabilities to the Top-100 configurations (bottom).
    
    \item \textbf{Low-Temperature Regime ($\beta=0.5$):} As $\beta$ increases beyond the critical point, the system is driven into an ordered phase characterized by strong spatial correlations. Figure~\ref{fig:ising_beta05} demonstrates the generator's successful adaptation to this highly coupled structure. The magnetization distribution (top-left) significantly broadens towards $\pm 1$, indicating the onset of $\mathbb{Z}_2$ symmetry breaking, while the energy distribution (top-right) correctly shifts towards lower-energy configurations.
\end{itemize}

Quantitatively, as reported in Table~\ref{tab:ising_merged_results}, the model achieves relative errors of less than $1.5\%$ across critical thermodynamic observables (Energy, Specific Heat, and Susceptibility) in both phases. Furthermore, the extremely low Total Variation (TV) errors confirm that the generator accurately learns the statistical weights of the exponentially large discrete state space without suffering from mode collapse.

\begin{figure}[H]
    \centering
    \includegraphics[width=0.8\linewidth]{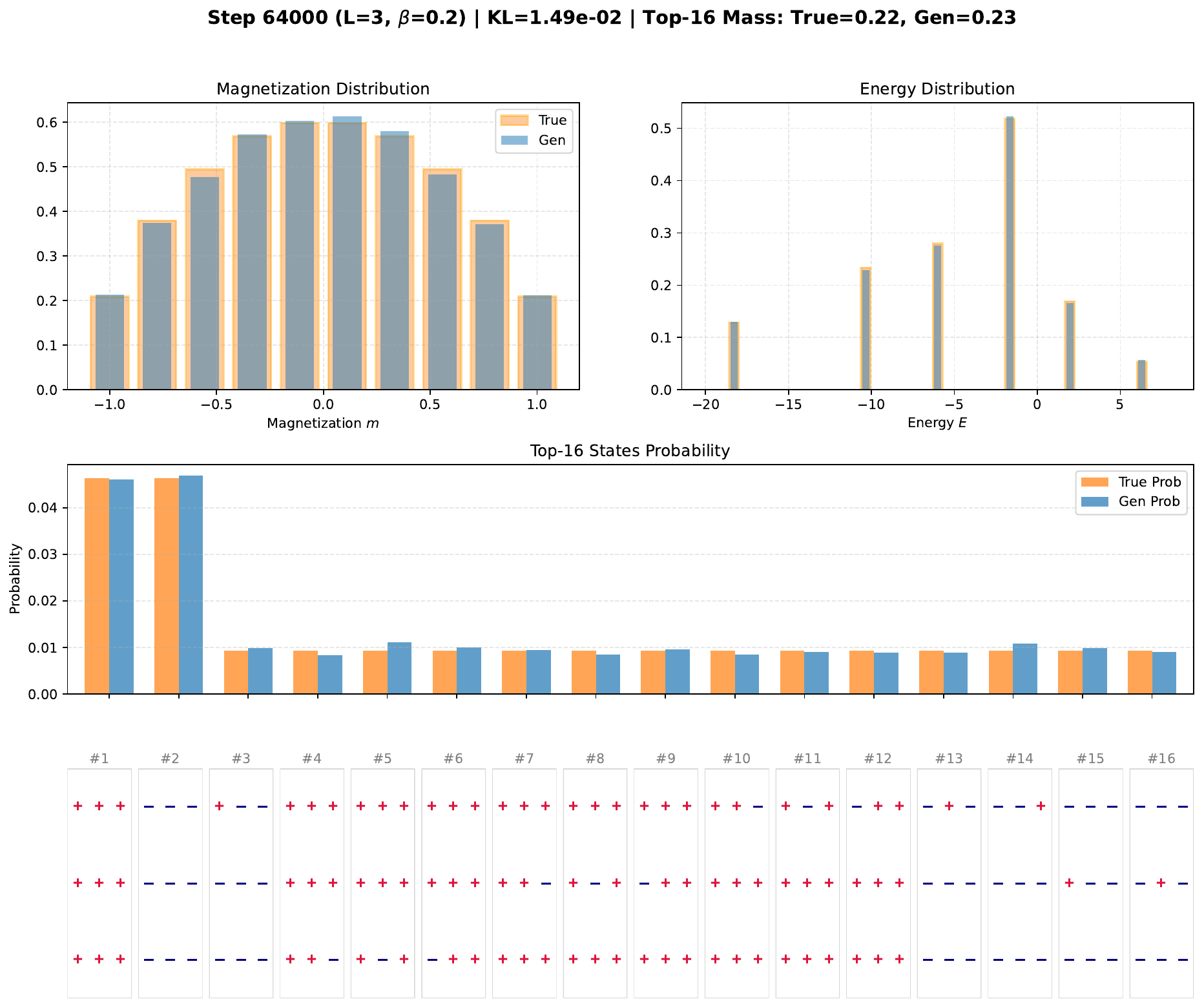}
    \caption{Generative fidelity for the 2D Ising Model in the high-temperature phase ($L=3, \beta=0.2$). \textbf{(Top-left)} Magnetization distribution showing a disordered state centered at zero. \textbf{(Top-right)} Energy distribution matching the theoretical spectrum. \textbf{(Bottom)} Probabilities of the Top-100 most likely configurations, demonstrating exact alignment with analytical enumeration.}
    \label{fig:ising_beta02}
\end{figure}

\begin{figure}[H]
    \centering
    \includegraphics[width=0.8\linewidth]{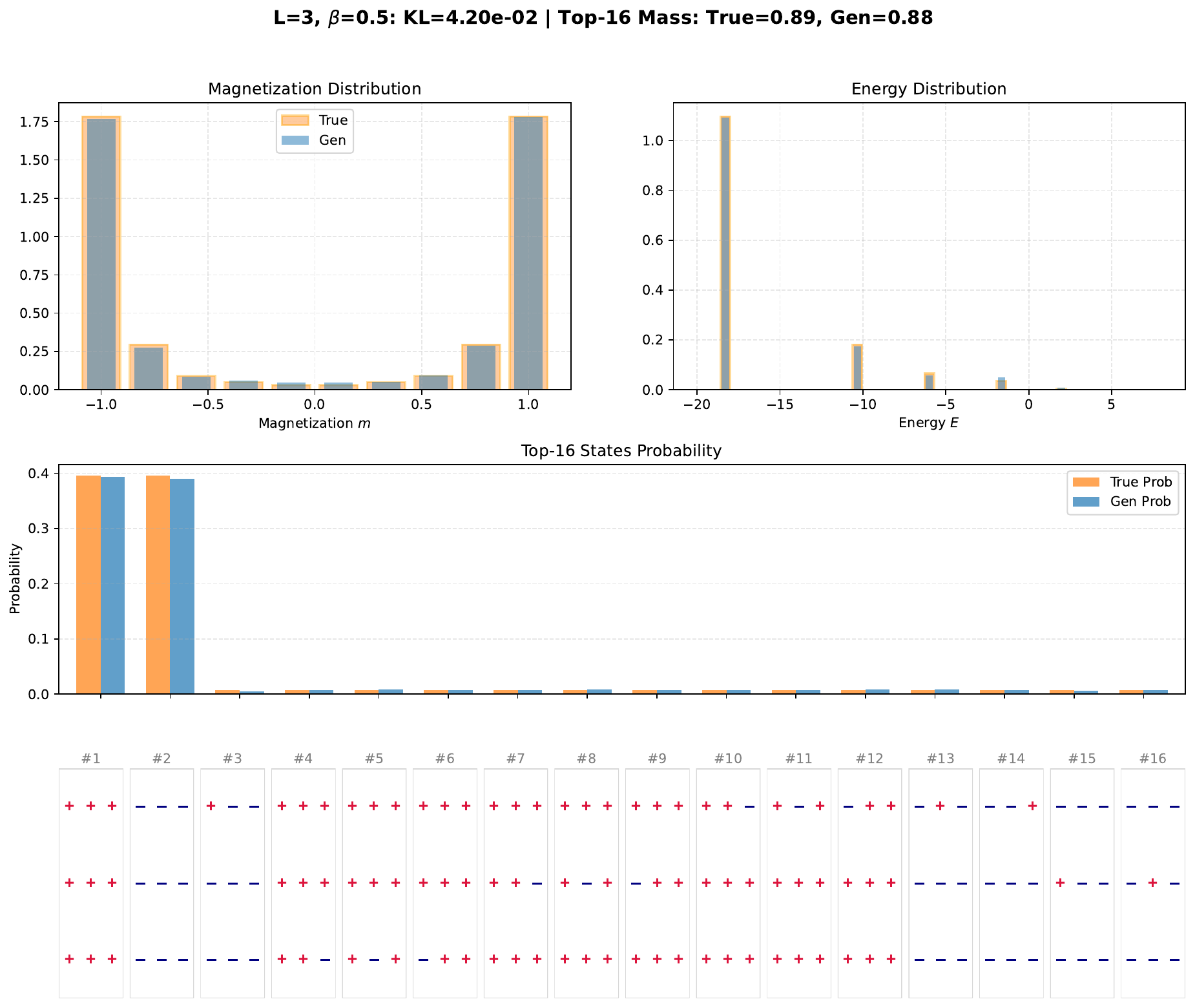}
    \caption{Generative fidelity for the 2D Ising Model in the low-temperature phase ($L=3, \beta=0.5$). \textbf{(Top-left)} Magnetization distribution capturing the broadened profile indicative of ordered states. \textbf{(Top-right)} Energy distribution properly shifted to lower states. \textbf{(Bottom)} Top-100 configuration probabilities, maintaining alignment despite stronger spin correlations.}
    \label{fig:ising_beta05}
\end{figure}

\begin{table}[H]
    \centering
    \caption{Quantitative analysis of generative performance on the 2D Ising model across different thermodynamic phases. Metrics are computed over 200,000 generated samples and compared against exact analytical values.}
    \label{tab:ising_merged_results}
    \begin{tabular}{@{}lccc@{}}
        \toprule
        \textbf{Metric} & \textbf{Ground Truth} & \textbf{Generated (Ours)} & \textbf{Error / Diff.} \\ 
        \midrule
        \multicolumn{4}{c}{\textit{High-Temperature Disordered Phase ($\beta=0.2$)}} \\
        \midrule
        \quad Energy $\langle E \rangle$                & $-4.8429$ & $-4.8211$ & $0.45\%$ (Rel.) \\
        \quad Abs Magnetization $\langle |m| \rangle$   & $0.4600$  & $0.4572$  & $0.0028$ (Abs.) \\
        \quad Specific Heat $C_v$                       & $1.3672$  & $1.3748$  & $0.56\%$ (Rel.) \\
        \quad Susceptibility $\chi$                     & $0.1486$  & $0.1493$  & $0.47\%$ (Rel.) \\
        \quad Total Variation (TV) Error                & $0.0000$  & $0.0331$  & $0.0331$ (Abs.) \\
        \midrule
        \multicolumn{4}{c}{\textit{Low-Temperature Ordered Phase ($\beta=0.5$)}} \\
        \midrule
        \quad Energy $\langle E \rangle$                & $-15.9091$ & $-15.7479$ & $1.01\%$ (Rel.) \\
        \quad Abs Magnetization $\langle |m| \rangle$   & $0.926$  & $0.919$  & $0.007$ (Abs.) \\
        \quad Specific Heat $C_v$                       & $4.677$  & $5.3921$  & $15.29\%$ (Rel.) \\
        \quad Susceptibility $\chi$                     & $0.1334$  & $0.1571$  & $17.70\%$ (Rel.) \\
        \quad Total Variation (TV) Error                & $0.0000$  & $0.0305$  & $0.0305$ (Abs.) \\
        \bottomrule
    \end{tabular}
\end{table}

%% file: section/conclusion.tex
\section{Conclusion}
\label{sec:conclusion}
In this work, we introduced a novel gradient-free generative framework capable of modeling complex statistical systems with mixed continuous and discrete state spaces. By formulating a reversible Maximum Mean Discrepancy (MMD) objective, our approach fundamentally bypasses the necessity for differentiable target densities and continuous computational graphs, a severe limitation that has long plagued Jacobian-based normalizing flows and score-based diffusion models in discrete domains. 

To achieve exact statistical coupling between disparate variable types, we proposed a split-head architecture alongside a bespoke Product Kernel that jointly measures similarities across domains. Our empirical evaluations demonstrate exceptional generative fidelity across multiple challenging thermodynamic regimes. On the purely discrete 2D Ising model, the framework accurately captures critical phase transition behaviors and thermal fluctuations without suffering from mode collapse. Furthermore, on the hybrid balanced double-well system, the model achieves near-perfect analytical alignment, seamlessly traversing high energy barriers to maintain equalized mode selection while flawlessly reconstructing the conditional continuous geometries.

Looking forward, this methodology opens highly promising avenues for generative modeling in the physical sciences. The inherent ability to jointly optimize discrete choices and continuous coordinates makes it uniquely suited for large-scale inverse problems, such as discovering metastable molecular conformations, designing combinatorial alloy structures, and exploring hybrid optimization landscapes where traditional gradient-based sampling catastrophically fails. We anticipate that integrating our reversible MMD framework with advanced Markov Chain Monte Carlo (MCMC) kernels will further accelerate its scalability towards high-dimensional, real-world physical simulations.